%% file: main.tex
\documentclass{article}

\usepackage[preprint]{neurips_2019}

\usepackage[utf8]{inputenc} 
\usepackage[T1]{fontenc}    
\usepackage{hyperref}       
\usepackage{url}            
\usepackage{booktabs}       
\usepackage{amsfonts}       
\usepackage{nicefrac}       
\usepackage{microtype}      

\usepackage{lipsum}
\usepackage{graphicx}
\usepackage{amsmath}
\usepackage{subfig}
\usepackage{amsthm}

\DeclareMathOperator*{\argmax}{arg\,max}
\DeclareMathOperator*{\argmin}{arg\,min}
\DeclareMathOperator*{\softmax}{softmax}

\DeclareMathOperator*{\logits}{logits}

\setcitestyle{numbers,comma,square}

\theoremstyle{plain}
\newtheorem{thm}{Theorem}
\newtheorem{lemma}[thm]{Lemma}

\title{Are Perceptually-Aligned Gradients a 
\\ General Property of Robust Classifiers?}

 \author{
  Simran Kaur \\
  Carnegie Mellon University \\
  \texttt{skaur@cmu.edu} \\
  \And
    Jeremy Cohen \\
    Carnegie Mellon University \\
  \texttt{jeremycohen@cmu.edu} \\
  \And
    Zachary C. Lipton \\
  Carnegie Mellon University \\
  \texttt{zlipton@cmu.edu} \\
}

\begin{document}
\maketitle

\begin{abstract}

\input{abstract.tex}
\end{abstract}

\section{Introduction}
\label{sec:intro}
\input{intro.tex}

\section{Experiments}
\label{sec:method}
\input{experiments.tex}

\bibliography{main}

\appendix

\newpage
\section{Additional images}
\label{sec:additional-images}

\input{appendix-images.tex}

\newpage
\section{Randomized Smoothing}
\label{sec:randomized-smoothing}
\input{appendix-randomized-smoothing.tex}

\newpage
\section{Details on Generating Images}
\label{sec:generating-details}
\input{appendix-objective-functions.tex}

\end{document}

%% file: abstract.tex
For a standard convolutional neural network, optimizing over the input pixels to maximize the score of some target class
will generally produce a grainy-looking version of the original image. 
However, Santurkar et al. (2019) demonstrated that for adversarially-trained neural networks,
this optimization produces images that uncannily resemble the target class.
In this paper, we show that these \emph{perceptually-aligned gradients} also occur 
under randomized smoothing, an alternative means of constructing adversarially-robust classifiers.
Our finding supports the hypothesis that perceptually-aligned gradients may be a general property of robust classifiers.
We hope that our results will inspire research aimed at explaining this link between perceptually-aligned gradients and adversarial robustness.

%% file: intro.tex

Classifiers are called \emph{adversarially robust} 
if they achieve high accuracy even on adversarially-perturbed inputs \cite{szegedy2014intriguing, biggio2013evasion}.
Two effective techniques for constructing robust classifiers are adversarial training and randomized smoothing.
In adversarial training, a neural network is optimized
via a min-max objective to achieve high accuracy on 
adversarially-perturbed training examples \cite{szegedy2014intriguing, kurakin2017adversarial, madry2017towards}.
In randomized smoothing, a neural network is smoothed by convolution with Gaussian noise \cite{lecuyer2018certified, li2018second, cohen2019certified, salman2019provably}.
Recently, \cite{tsipras2018robustness, santurkar2019image, engstrom2019learning} demonstrated 
that adversarially-trained networks exhibit \emph{perceptually-aligned gradients}: iteratively updating an image by gradient ascent so as to maximize the score assigned to a target class will render an image that perceptually resembles 
the target class.

In this paper, we show that smoothed neural networks also exhibit perceptually-aligned gradients.
This finding supports the conjecture in \cite{tsipras2018robustness, santurkar2019image, engstrom2019learning} that perceptually-aligned gradients 
may be a general property of robust classifiers,
and not only a curious consequence of adversarial training.
Since the root cause behind the apparent relationship between adversarial robustness and perceptual alignment remains unclear, we hope that our findings will spur foundational research aimed at explaining this connection.

\paragraph{Perceptually-aligned gradients}
Let $f: \mathbb{R}^d \to \mathbb{R}^k$ be a neural network image classifier 
that maps from images in $\mathbb{R}^d$ to scores for $k$ classes.
Naively, one might hope that 
by starting with any image $\mathbf{x}_0 \in \mathbb{R}^d$
and taking gradient steps so as to maximize the score of a target class $t \in [k]$,
we would produce an altered image 
that better resembled (perceptually) the targeted class. 
However, as shown in Figure \ref{fig:three-models-compared}, 
when $f$ is a vanilla-trained neural network, this is not the case; 
iteratively following the gradient of class $t$'s score 
appears perceptually as a \emph{noising} of the image. 
In the nascent literature on the explainability of deep learning, 
this problem has been addressed by adding explicit regularizers 
to the optimization problem \cite{olah2017feature, nguyen2014deep, mahendran2015understanding, oygard2015visualizing}.
However, \cite{santurkar2019image} showed that for adversarially-trained neural networks, 
these explicit regularizers aren't needed --- merely following the gradient of a target class $t$ 
will render images that visually resemble class $t$. 

\begin{figure}[t]
  \centering
  \includegraphics[scale=0.25]{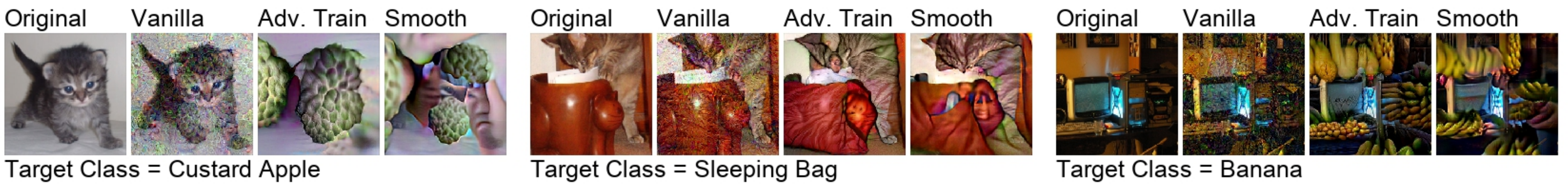}
  \caption{Large-$\epsilon$ targeted adversarial examples for a vanilla-trained network, 
  an adversarially trained network, and a smoothed network.  
  Adversarial examples for both robust classifiers visually resemble the targeted class, while adversarial examples for the vanilla classifier do not.
  All of these adversarial examples have perturbation size $\epsilon = 40$ (on images with pixels scaled to $[0, 1]$).}
  \label{fig:three-models-compared}
\end{figure}




\paragraph{Randomized smoothing}
Across many studies, adversarially-trained neural networks have proven empirically successful at resisting adversarial attacks
within the threat model in which they were trained \cite{athalye2018obfuscated, brendel2019accurate}.
Unfortunately, when the networks are large and expressive,
no known algorithms are able to provably \emph{certify} 
this robustness \cite{salman2019convex}, 
leaving open the possibility that they will be vulnerable to better adversarial 
attacks developed in the future.

For this reason, a distinct approach to robustness called \emph{randomized smoothing} has recently gained traction in the literature \cite{lecuyer2018certified, li2018second, cohen2019certified, salman2019provably}.
In the $\ell_2$-robust version of randomized smoothing, the robust classifier $\hat{f}_\sigma: \mathbb{R}^d \to \mathbb{R}^k$ 
is a \emph{smoothed neural network} of the form:
\begin{align}
        \hat{f}_\sigma(\mathbf{x}) = \mathbb{E}_{\boldsymbol{\varepsilon} \sim \mathcal{N}(0, \sigma^2 I)}[ f(\mathbf{x} + \boldsymbol{\varepsilon})]
        \label{eq:smoothing}
\end{align}
where $f: \mathbb{R}^d \to \mathbb{R}^k$ is a neural network (ending in a softmax) called the \emph{base network}.
In other words, $\hat{f}_\sigma(\mathbf{x})$, the smoothed network's predicted scores at $\mathbf{x}$,
is the weighted average of $f$
within the neighborhood around $\mathbf{x}$, 
where points are weighted according to an isotropic Gaussian
centered at $\mathbf{x}$ with variance $\sigma^2$.
A disadvantage of randomized smoothing is that 
the smoothed network $\hat{f}_\sigma$ cannot be evaluated exactly, due to the expectation in \eqref{eq:smoothing},
and instead must approximated via Monte Carlo sampling.
However, by computing $\hat{f}_\sigma(\mathbf{x})$ one can obtain a
guarantee that $\hat{f}_\sigma$'s prediction is constant
within an $\ell_2$ ball around $\mathbf{x}$; in contrast, it is not currently possible to obtain such certificates using neural network classifiers.
See Appendix \ref{sec:randomized-smoothing} for more background on randomized smoothing.

How to best train the base network $f$ to maximize the certified accuracy of the smoothed network $\hat{f}_\sigma$
remains an open question in the literature.
In \cite{lecuyer2018certified, cohen2019certified}, 
the base network $f$ was trained with Gaussian data augmentation.
However, \cite{carmon2019unlabeled, li2018second} showed 
that training $f$ instead using stability training \cite{zheng2016improving} 
resulted in substantially higher certified accuracy, 
and \cite{salman2019provably} showed that training $f$ by adversarially training $\hat{f}_\sigma$ also outperformed Gaussian data augmentation.
Our main experiments use a base network trained with Gaussian data augmentation.
In Appendix \ref{sec:generating-details} we compare against the network from \cite{salman2019provably}.

%% file: experiments.tex
\begin{figure}[b]
  \centering
  \includegraphics[scale=0.4]{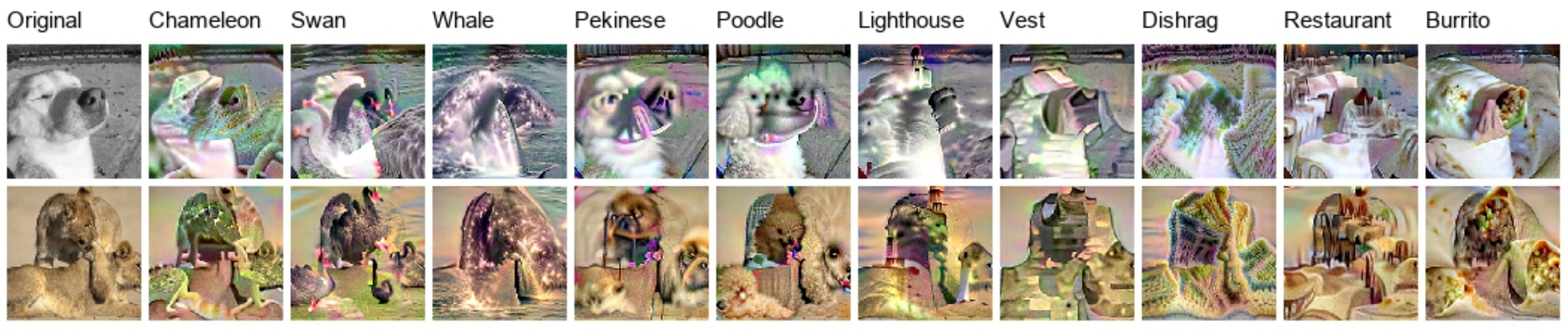}
  \caption{Large-$\epsilon$ adversarial examples for a smoothed neural network. 
  Each row is a (random) starting image, each column is a (random) target class.
  See Figures \ref{fig:all-large-ep-adversarial-ex-1}-\ref{fig:all-large-ep-adversarial-ex-2} in Appendix \ref{sec:additional-images} for more.
  }
  \label{fig:large-epsilon}
\end{figure}

In this paper, we show that smoothed neural networks exhibit perceptually-aligned gradients.
By design, our experiments mirror those conducted in \cite{santurkar2019image}.
To begin, we synthesize large-$\epsilon$ targeted adversarial examples 
for a smoothed  ($\sigma=0.5)$ ResNet-50 trained on ImageNet \cite{he2016deep, imagenetcvpr09}.
Given some source image $\mathbf{x}_0$, we used projected gradient descent (PGD)
to find an image $\mathbf{x}^*$ within $\ell_2$ distance $\epsilon$ of $\mathbf{x}_0$ 
that the smoothed network $\hat{f}_\sigma$ 
classifies confidently as target class $t$.
Specifically, decomposing $f$ as $f(\mathbf{x}) = \softmax(\logits(\mathbf{x}))$, we solve the problem:
\begin{align}
    \mathbf{x}^* = \argmax_{\mathbf{x}: \; \|\mathbf{x} - \mathbf{x}_0 \| \le \epsilon} \mathbb{E}_{\boldsymbol{\varepsilon} \sim \mathcal{N}(0, \sigma^2 I)}[\logits(\mathbf{x} + \boldsymbol{\varepsilon})_t]
    \label{eq:obj}.
\end{align}
We find that optimizing \eqref{eq:obj} yields visually 
more compelling results than minimizing the cross-entropy loss of $\hat{f}_\sigma$.
See Appendix \ref{sec:generating-details} for a comparison between \eqref{eq:obj} and the cross-entropy approach.

The gradient of the objective \eqref{eq:obj} cannot be computed exactly, 
due to the expectation over $\boldsymbol{\varepsilon}$, 
so we instead used an unbiased estimator obtained 
by sampling $N=20$ noise vectors
$\boldsymbol{\varepsilon}_1, \hdots, \boldsymbol{\varepsilon}_N \sim \mathcal{N}(0, \sigma^2I)$ 
and computing the average gradient 
$\frac{1}{N} \sum_{i=1}^N \nabla_{\mathbf{x}} \logits(\mathbf{x} + \boldsymbol{\varepsilon}_i)_t$.

Figure \ref{fig:three-models-compared} depicts large-$\epsilon$ targeted adversarial examples 
for a vanilla-trained neural network, an adversarially trained network \cite{madry2017towards}, and a smoothed network.
Observe that the adversarial examples for the vanilla network 
do not take on coherent features of the target class,
while the adversarial examples for both robust networks do.
Figure \ref{fig:large-epsilon} shows large-$\epsilon$ targeted adversarial examples 
synthesized for the smoothed network for a variety of different target classes.

Next, as in \cite{santurkar2019image}, we use the smoothed network 
to class-conditionally synthesize images.
To generate an image from class $t$, we sample a seed image $\mathbf{x}_0$ from a multivariate Gaussian fit to images from class $t$,
and then we iteratively take gradient steps 
to maximize the score of class $t$ using objective \eqref{eq:obj}.
Figure \ref{fig:generation6} shows two images synthesized in this way from each of seven ImageNet classes.
The synthesized images appear visually similar 
to instances of the target class, though they often lack global coherence --- the 
synthesized solar dish includes multiple overlapping solar dishes.

\begin{figure}
  \centering
  \includegraphics[scale=0.4]{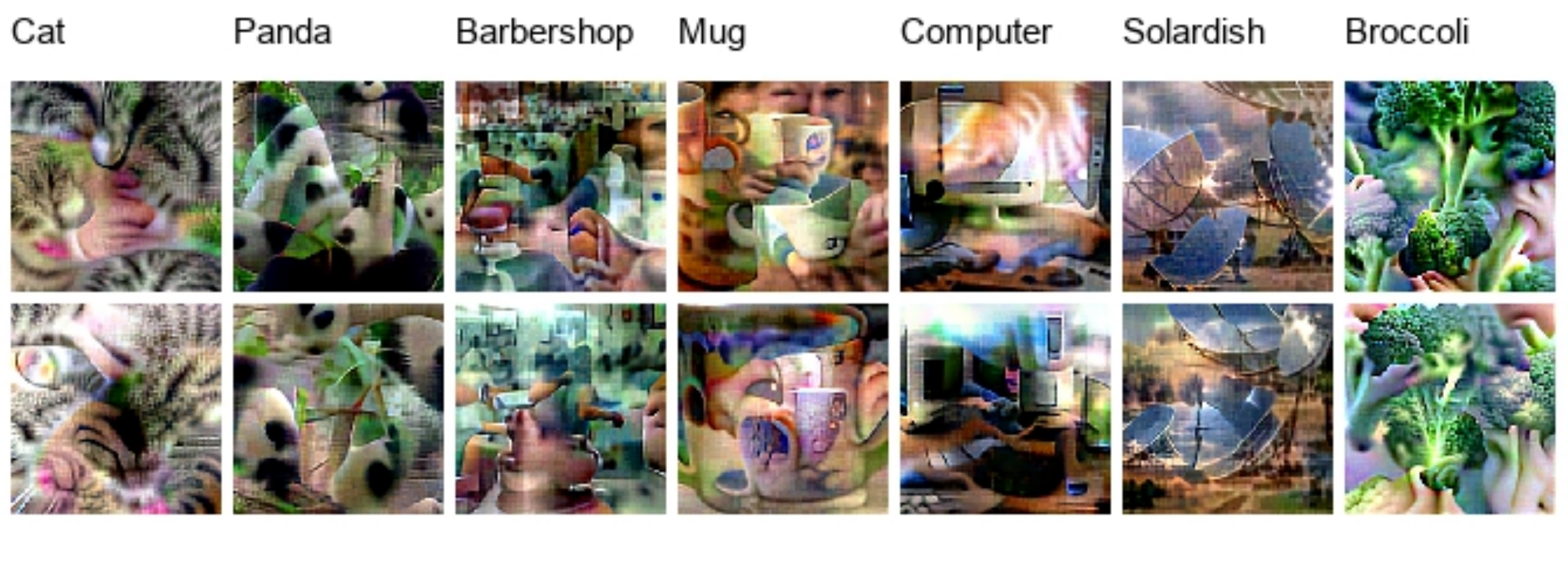}
  \caption{Class-conditional image synthesis using a smoothed NN.  
  To synthesize an image from class $t$, we sampled a seed image from a multivariate Gaussian fit to images from class $t$, and then performed PGD to maximize the score of class $t$.  Figures \ref{fig:all-gaussian-synthesized-1}-\ref{fig:all-gaussian-synthesized-2} in Appendix \ref{sec:additional-images} have more examples.} 
  \label{fig:generation6}
\end{figure}

\paragraph{Noise Level $\sigma$}
Smoothed neural networks have a hyperparameter $\sigma$ which controls a robustness/accuracy tradeoff: 
when $\sigma$ is high, the smoothed network is more robust, 
but less accurate \cite{lecuyer2018certified, cohen2019certified}.
We investigated the effect of $\sigma$ on the perceptual quality of generated images.
Figure \ref{fig:noise-level} shows large-$\epsilon$ adversarial examples crafted for smoothed networks 
with $\sigma$ varying in $\{0.25, 0.50, 1.00\}$.
Observe that when $\sigma$ is large, PGD tends to paint single instance of the target class; 
when $\sigma$ is small, PGD tends to add spatially scattered features.

\begin{figure}[b!]
  \centering
  \includegraphics[scale=0.4]{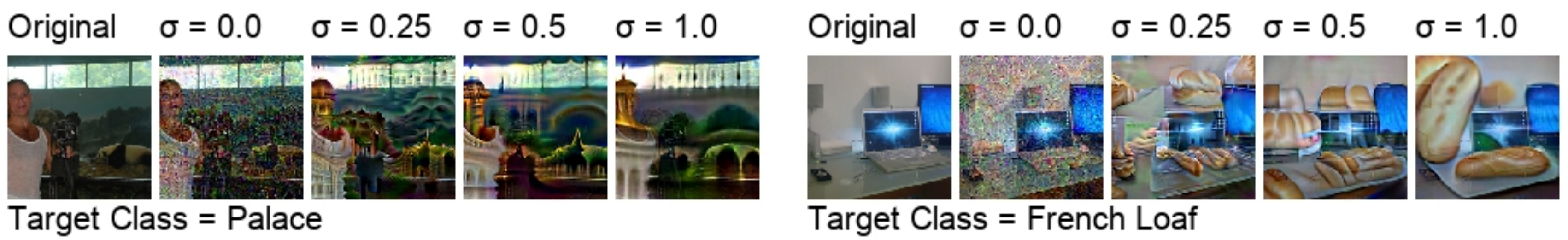}
  \caption{Large-$\epsilon$ adversarial examples crafted for smoothed neural networks 
  with different settings of the smoothing scale hyperparameter $\sigma$.
  More examples are in Figures \ref{fig:noise-level-variation-pt1}-\ref{fig:noise-level-variation-pt3} in Appendix \ref{sec:additional-images}.
}
  \label{fig:noise-level}
\end{figure}

\paragraph{Other concerns}
In Appendix \ref{sec:generating-details}, we study the effects of the following factors 
on the perceptual quality of the generated images: 
the number of Monte Carlo noise samples $N$, the loss function used for PGD, 
and whether the base network $f$ is trained using Gaussian data augmentation 
\cite{lecuyer2018certified, cohen2019certified} or \textsc{SmoothAdv} \cite{salman2019provably}.

%% file: appendix-images.tex
\begin{figure}[!h]
\centering
  \includegraphics[scale=0.5]{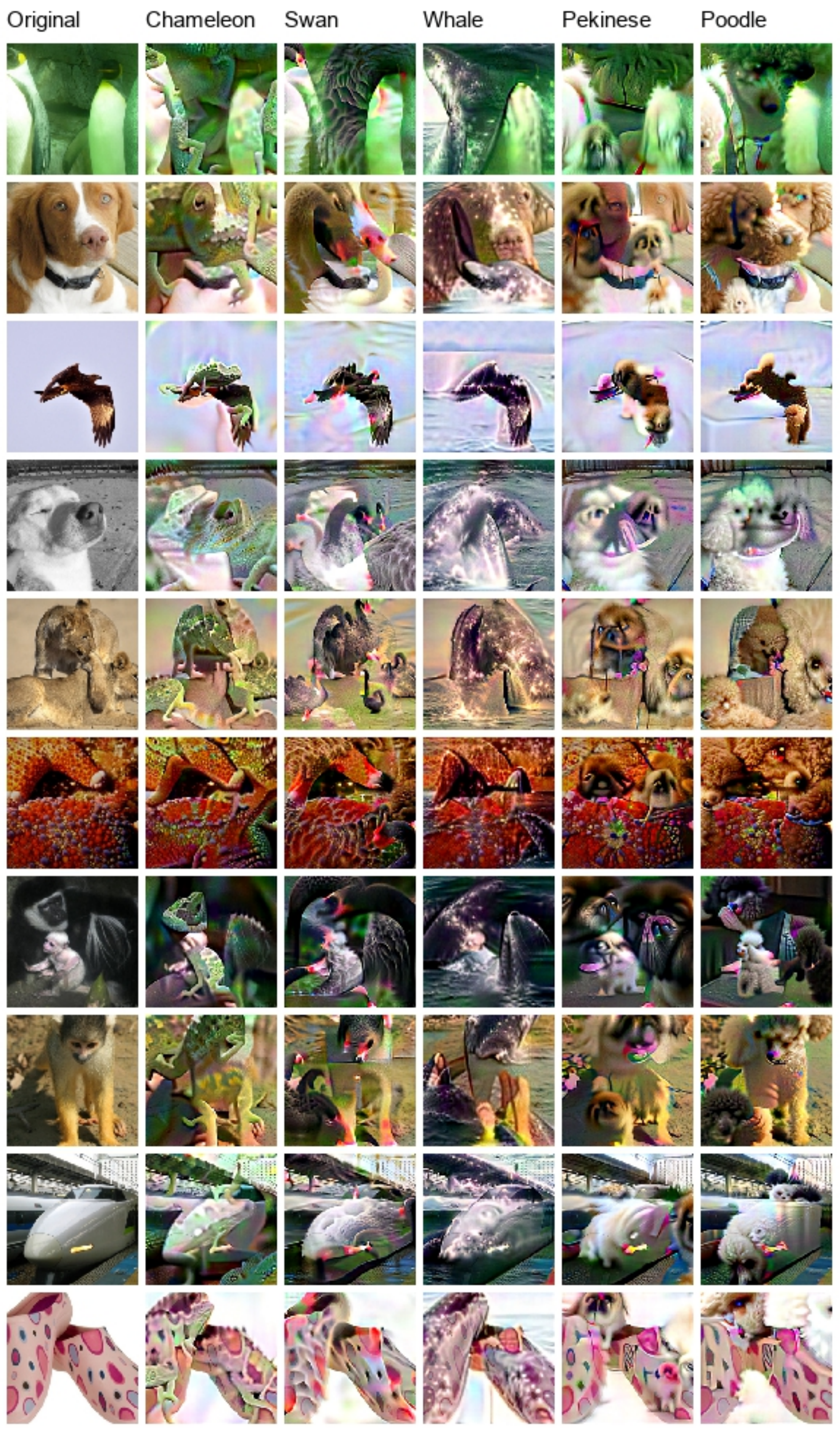}
  \caption{Large-$\epsilon$ adversarial examples for a smoothed neural network (part 1 / 2). 
  Each row is a randomly chosen starting image, each column is a randomly chosen target class.}
  \label{fig:all-large-ep-adversarial-ex-1}
\end{figure}

\newpage

\begin{figure}[!h]
\centering
  \includegraphics[scale=0.5]{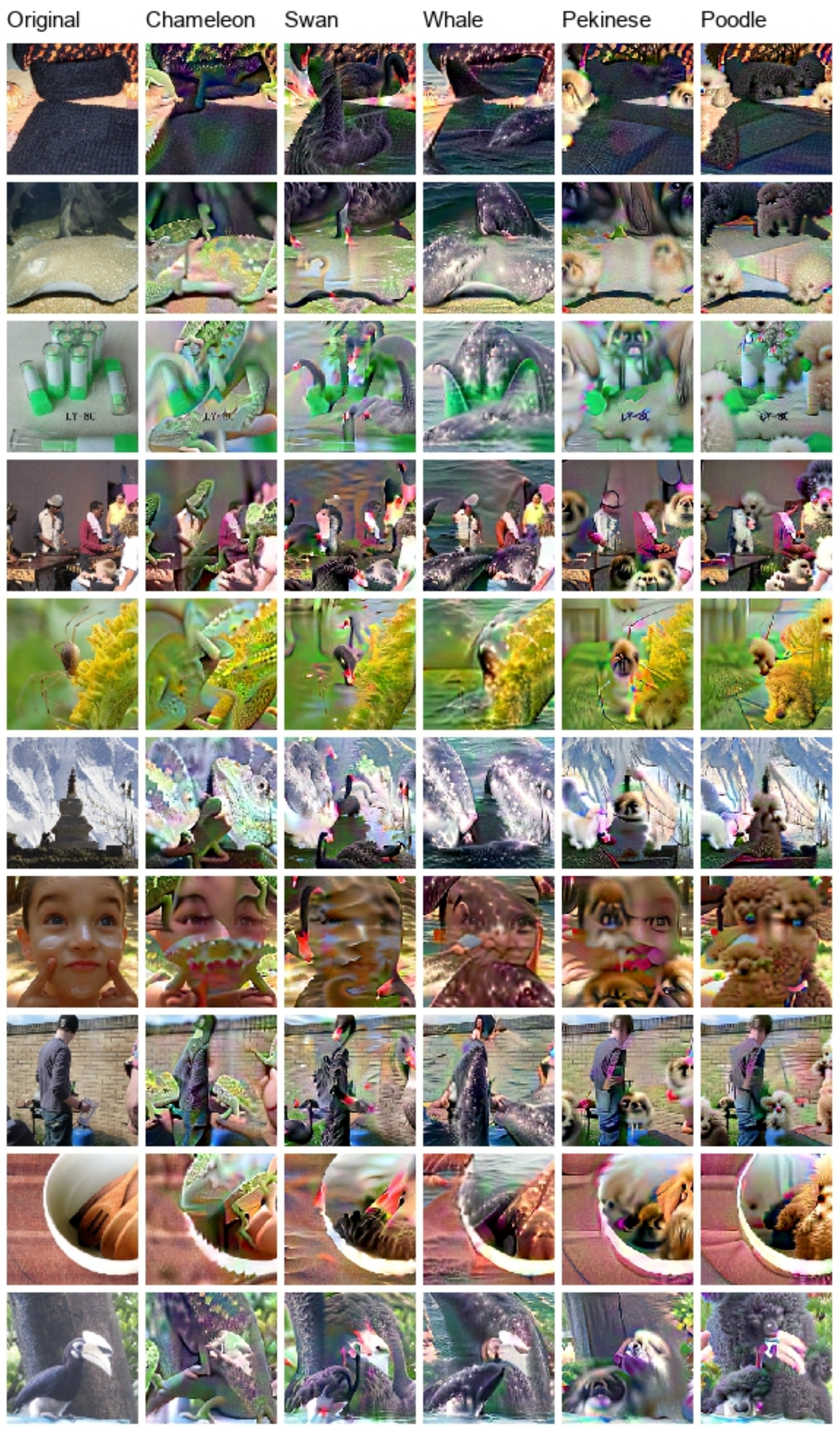}
  \caption{Large-$\epsilon$ adversarial examples for a smoothed neural network (part 2 / 2). 
  Each row is a randomly chosen starting image, each column is a randomly chosen target class.}
  \label{fig:all-large-ep-adversarial-ex-2}
\end{figure}

\newpage

\begin{figure}[!h]
\centering
  \subfloat[cat]{\includegraphics[scale=0.5]{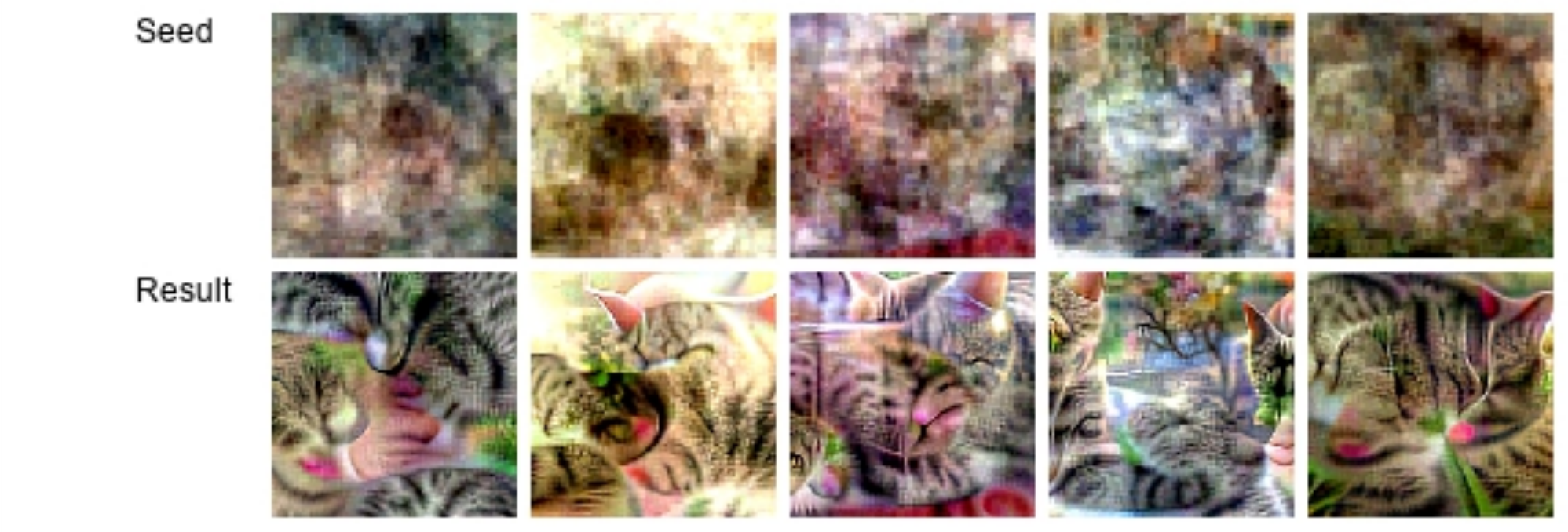}} \quad

  \subfloat[panda]{\includegraphics[scale=0.5]{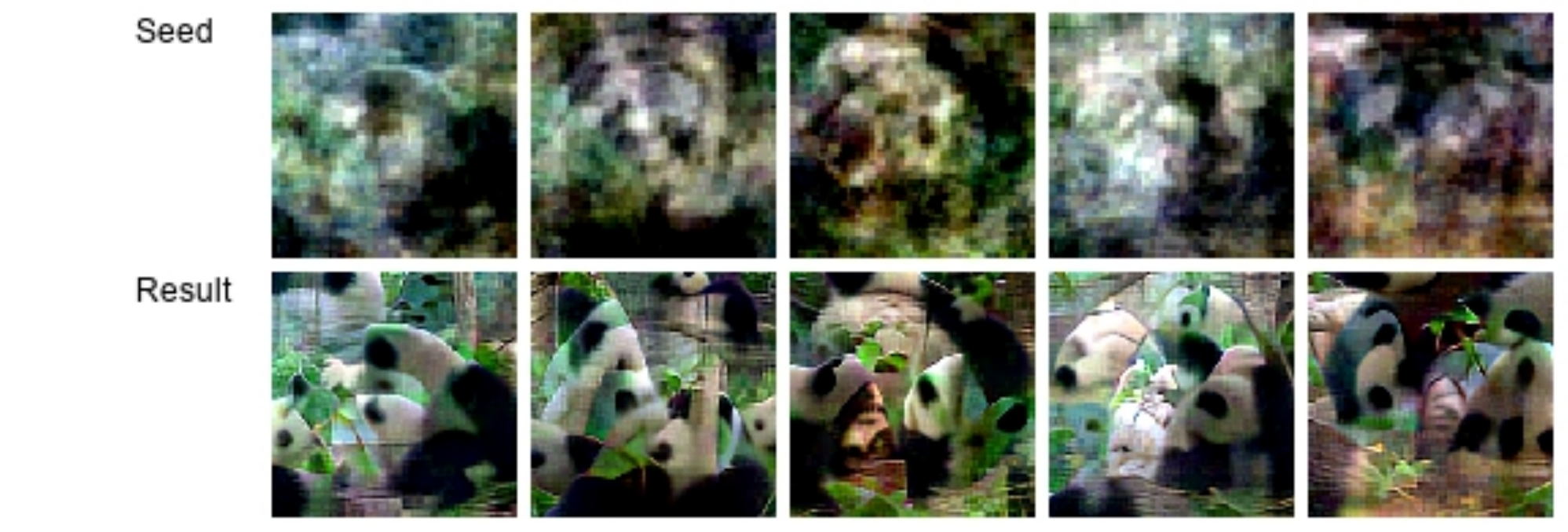}} \quad

  \subfloat[barber shop]{\includegraphics[scale=0.5]{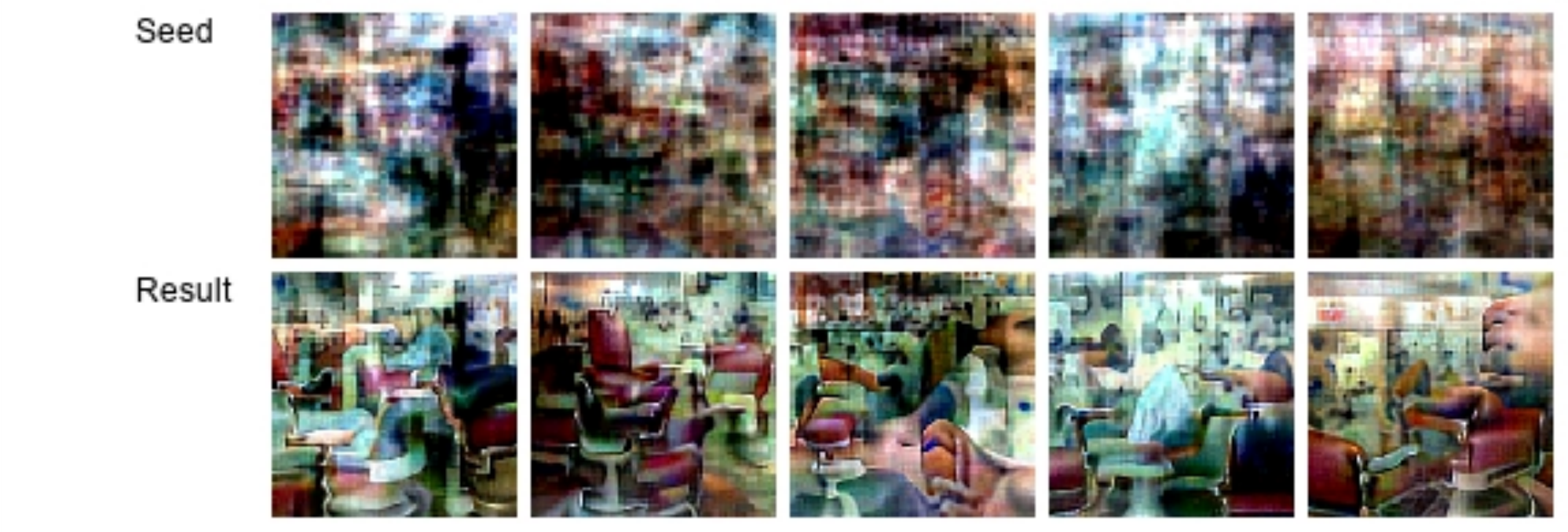}}\quad

  \subfloat[mug]{\includegraphics[scale=0.5]{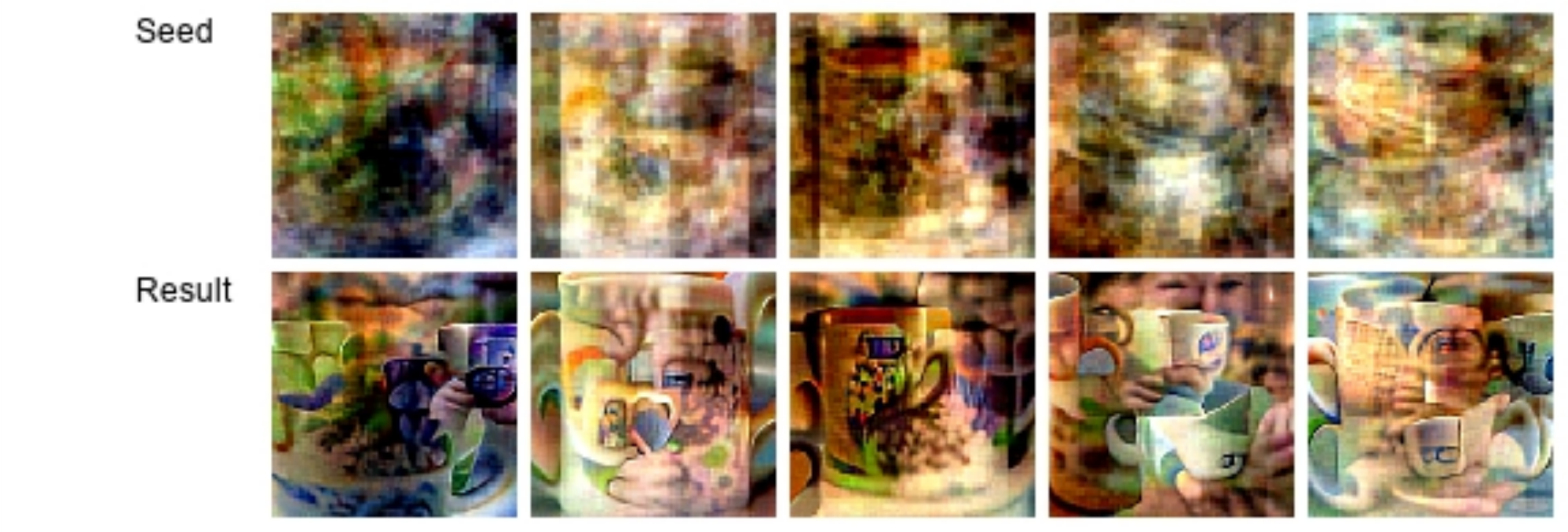}} \quad
  \caption{Class-conditional synthesized images (part 1 / 2).  
  To synthesize an image from class $t$, we sampled a seed image from a multivariate Gaussian distribution fit to class $t$, and then performed PGD to maximize the score which a smoothed neural network assigns to class $t$.  The top row shows the seed image, the bottom row shows the result of PGD.}
  \label{fig:all-gaussian-synthesized-1}
\end{figure}

\newpage

\begin{figure}[!h]
\centering
  \subfloat[computer]{\includegraphics[scale=0.5]{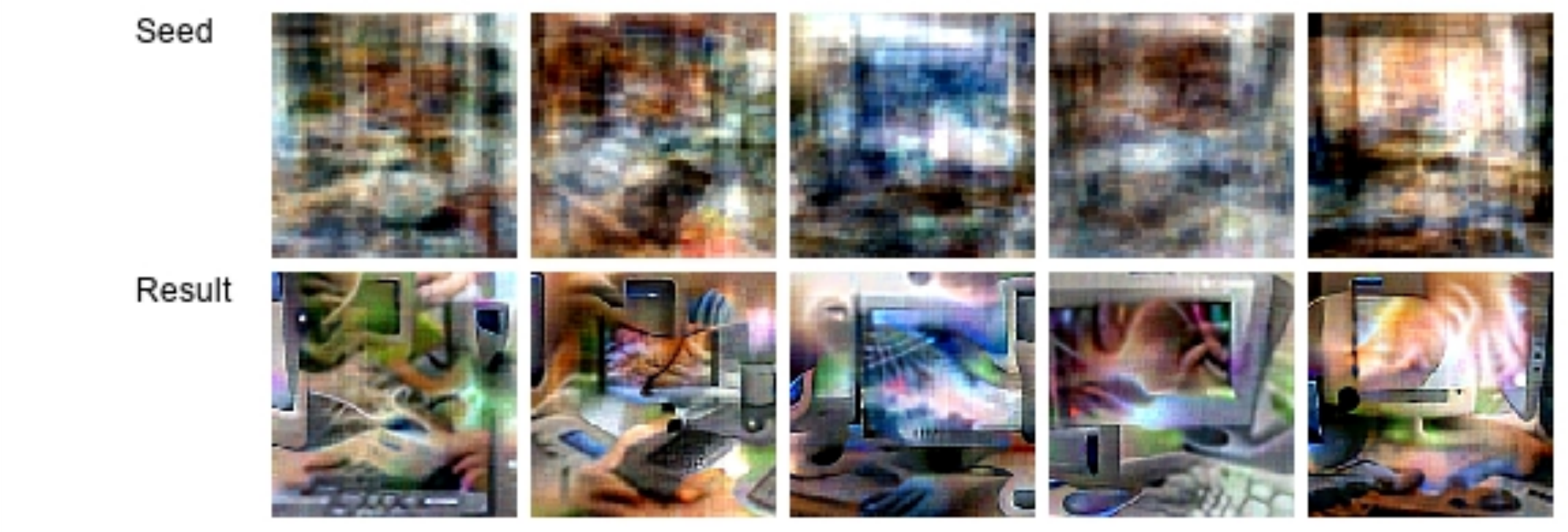}} \quad

  \subfloat[solar dish]{\includegraphics[scale=0.5]{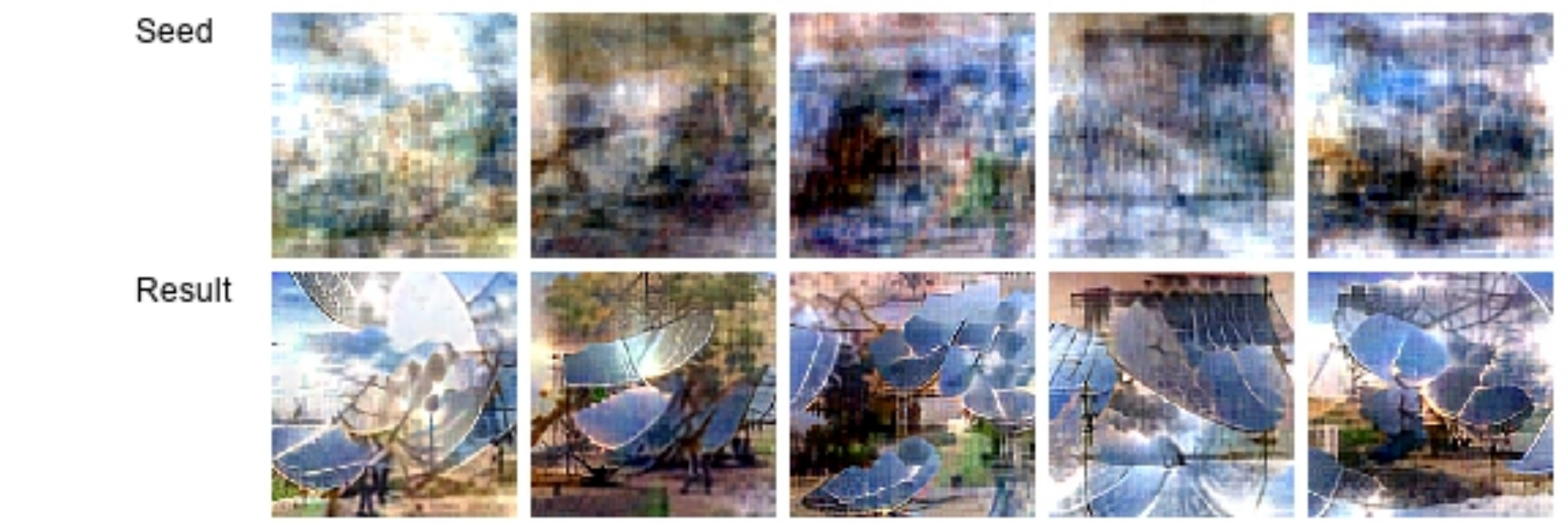}}\quad

  \subfloat[broccoli]{\includegraphics[scale=0.5]{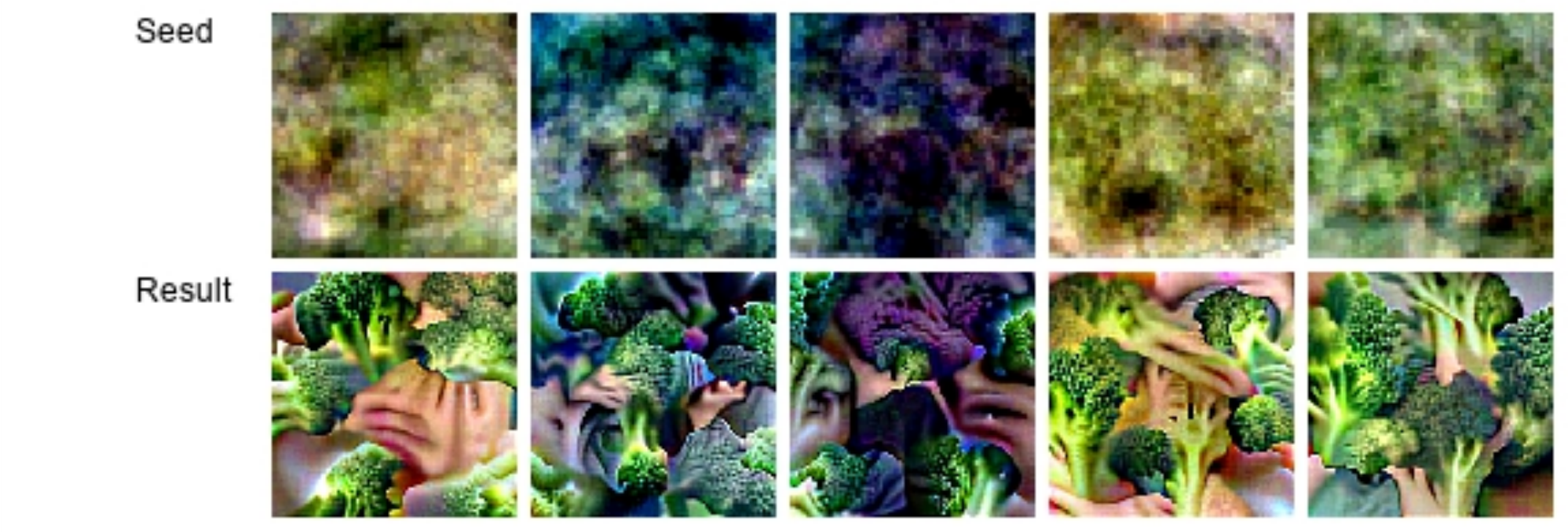}} \quad 
  \caption{Class-conditional synthesized images (part 2 / 2).  
  To synthesize an image from class $t$, we sampled a starting image from a multivariate Gaussian distribution fit to class $t$, and then performed PGD to maximize the score which a smoothed neural network assigns to class $t$. The top row shows the seed image, the bottom row shows the result of PGD.}
  \label{fig:all-gaussian-synthesized-2}
\end{figure}

\newpage

\begin{figure}[!h]
\centering
  \includegraphics[scale=0.5]{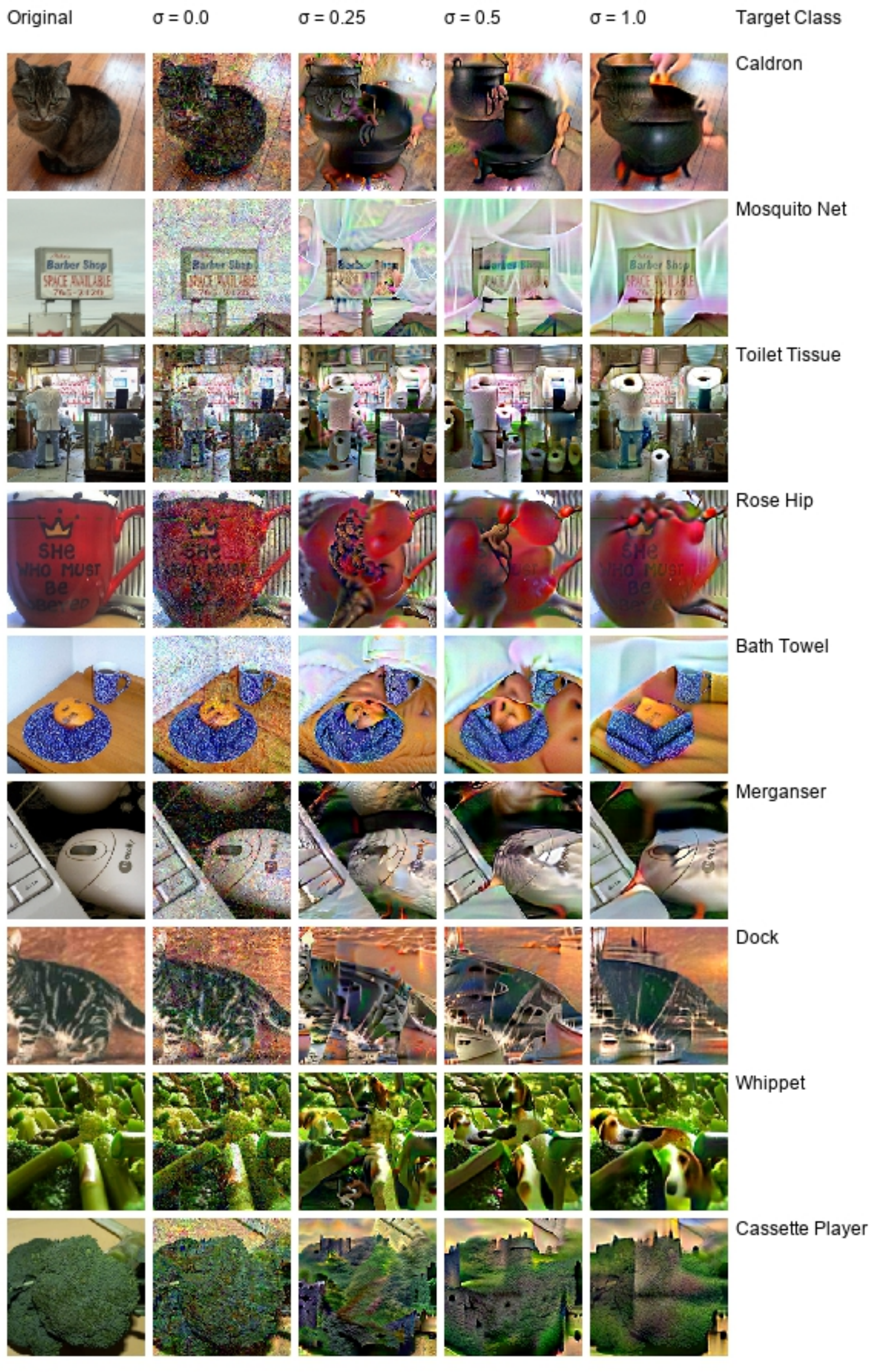}
  \caption{Large-$\epsilon$ adversarial examples crafted for smoothed neural networks with different settings of the smoothing scale hyperparameter $\sigma$ (part 1 / 3).  Images and target classes were randomly chosen.  When $\sigma$ is large, the adversary tends to paint a single, coherent instance of the target class; when $\sigma$ is small,  the adversary tends to paint scattered features of the target class.  Note that $\sigma=0.0$ corresponds to a vanilla-trained network. }
  \label{fig:noise-level-variation-pt1}
\end{figure}

\newpage

\begin{figure}[!h]
\centering
  \includegraphics[scale=0.5]{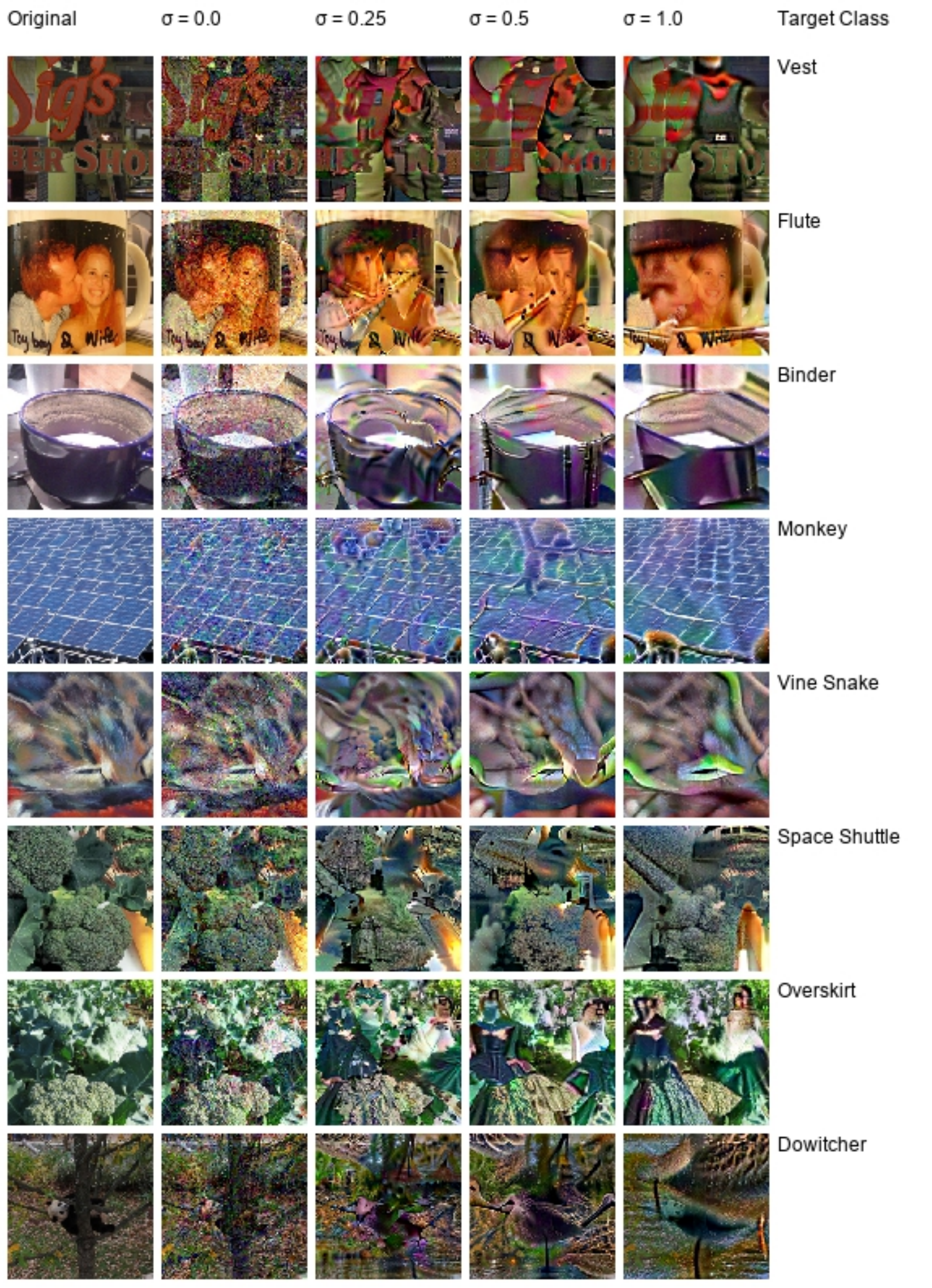}
  \caption{Large-$\epsilon$ adversarial examples crafted for smoothed neural networks 
  with different settings of the smoothing scale hyperparameter $\sigma$ (part 2 / 3).  Images and target classes were randomly chosen.  When $\sigma$ is large, the adversary tends to paint a single, coherent instance of the target class; when $\sigma$ is small,  the adversary tends to paint scattered features of the target class. Note that $\sigma=0.0$ corresponds to a vanilla-trained network.}
  \label{fig:noise-level-variation-pt2}
\end{figure}

\newpage

\begin{figure}[!h]
\centering
  \includegraphics[scale=0.5]{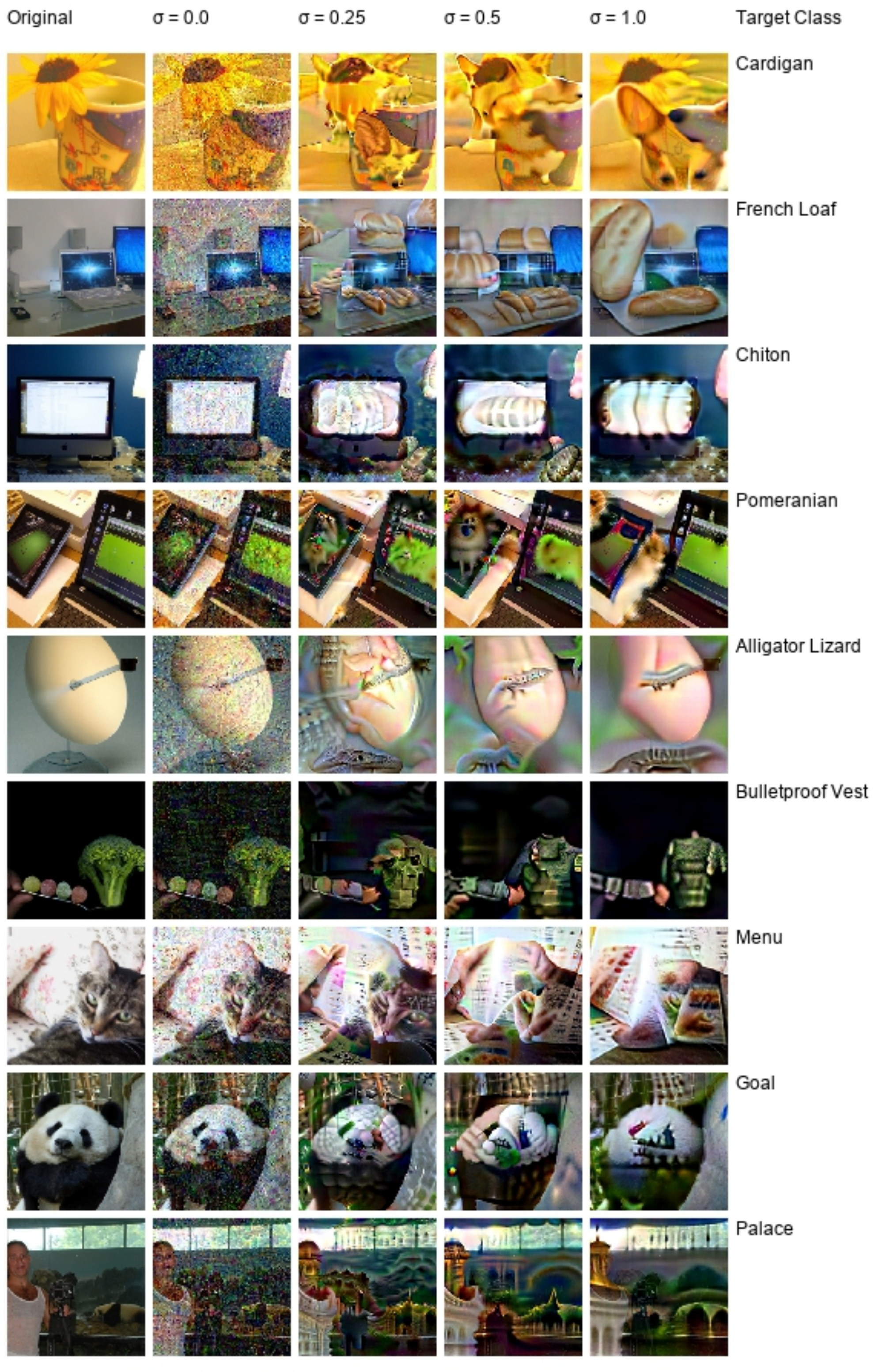}
  \caption{Large-$\epsilon$ adversarial examples crafted for smoothed neural networks 
  with different settings of the smoothing scale hyperparameter $\sigma$ (part 3 / 3).  Images and target classes were randomly chosen.  When $\sigma$ is large, the adversary tends to paint a single, coherent instance of the target class; when $\sigma$ is small,  the adversary tends to paint scattered features of the target class. Note that $\sigma=0.0$ corresponds to a vanilla-trained network.}
  \label{fig:noise-level-variation-pt3}
\end{figure}

\newpage

\begin{figure}[!h]
  \centering
  \includegraphics[scale=0.5]{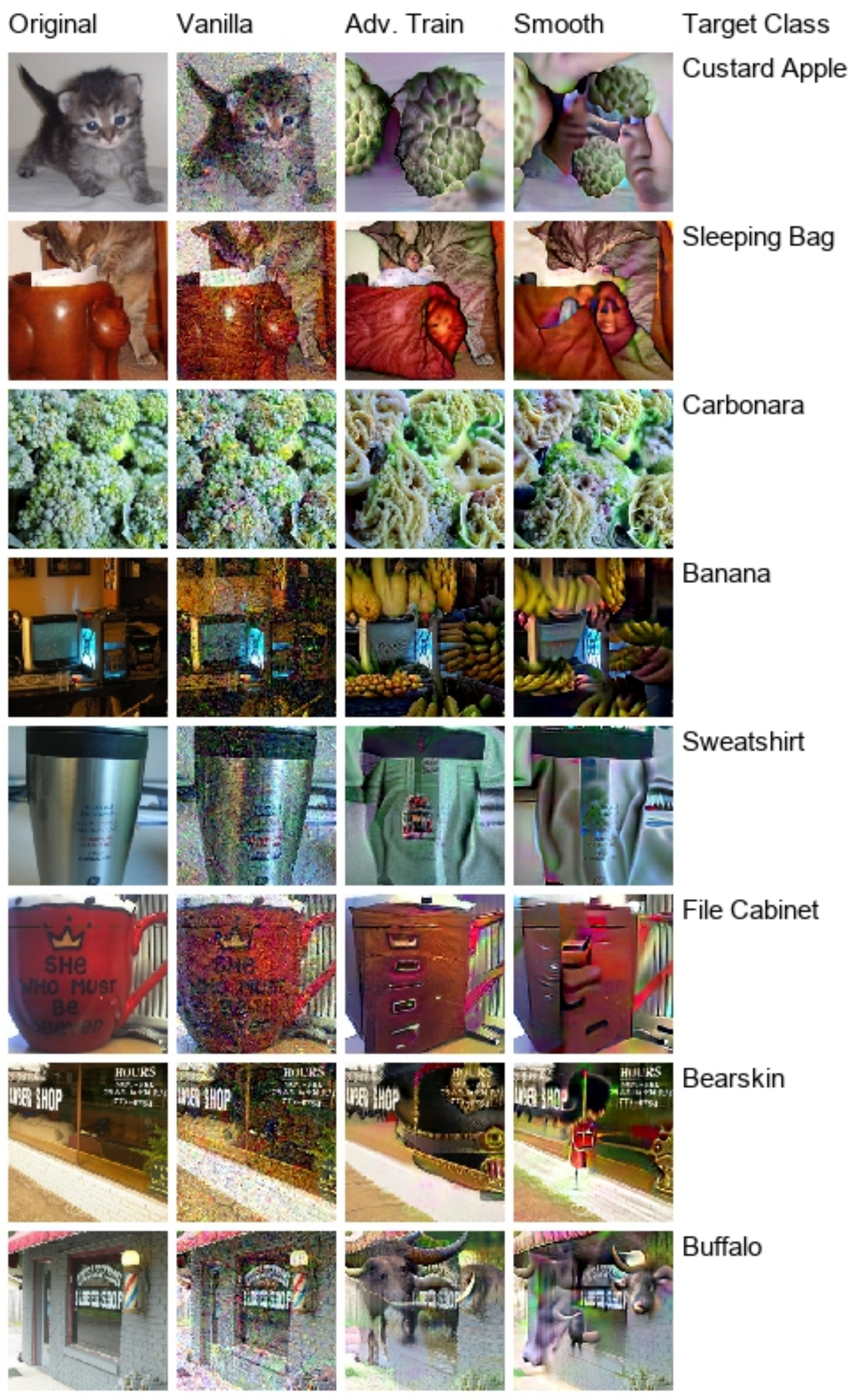}
  \caption{Large-$\epsilon$ targeted adversarial examples for a vanilla-trained network, 
  an adversarially trained network \protect\cite{madry2017towards}, and a smoothed network.  
  Adversarial examples for both robust classifiers visually resemble the targeted class, while adversarial examples for the vanilla classifier do not.
  All of these adversarial examples have perturbation size $\epsilon = 40$ (on images with pixels scaled to $[0, 1]$).}
  \label{fig:three-models-compared-all}
\end{figure}

%% file: appendix-randomized-smoothing.tex
Randomized smoothing is relatively new to the literature, and few comprehensive references exist.
Therefore, in this appendix, we review some basic aspects of the technique.

\paragraph{Preliminaries}
Randomized smoothing refers to a class of adversarial defenses in which the robust classifier $g: \mathbb{R}^d \to [k]$ that maps from an input in $\mathbb{R}^d$ to a class in $[k] := \{1, \hdots, k\}$ is defined as:
$$ g(\mathbf{x}) = \argmax_{y \in [k]} \; \mathbb{E}_T[ f(T(\mathbf{x})) ]_y  .$$
Here, $f: \mathbb{R}^d \to \Delta_k$ is a neural network ``base classifier'' which maps from an input in $\mathbb{R}^d$ to a vector of class scores in $\Delta_k := \{\mathbf{z} \in \mathbb{R}^k: \mathbf{z} \ge 0, \sum_{j=1}^k z_j = 1\}$, the probability simplex of non-negative $k$-vectors that sum to 1.   $T$ is a randomization operation which randomly corrupts inputs in $\mathbb{R}^d$ to other inputs in $\mathbb{R}^d$, i.e. for any $\mathbf{x}$, $T(\mathbf{x})$ is a random variable.

Intuitively, the score which the smoothed classifier $g$ assigns to class $y$ for the input $\mathbf{x}$ is defined to be the \emph{expected} score that the base classifier $f$ assigns to the class $y$ for the random input $T(\mathbf{x})$.

The requirement that $f$ returns outputs in the probability simplex $\Delta_k$ can be satisfied in either of two ways.
In the ``soft smoothing'' formulation (presented in the main paper), $f$ is a neural network which ends in a softmax.
In the ``hard smoothing'' formulation, $f$ returns the indicator vector for a particular class, i.e. a length-$k$ vector with one 1 and the rest zeros, without exposing the intermediate class scores.
In the hard smoothing formulation, since the expectation of an indicator function is a probability, the smoothed classifier $g(\mathbf{x})$ can be interpreted as returning the most probable prediction by the classifier $f$ over the random variable $T(\mathbf{x})$.
Note that no papers have yet studied soft smoothing as a certified defense, though \cite{salman2019provably} approximated a hard smoothing classifier with the corresponding soft classifier in order to attack it. 

When the base classifier $f$ is a neural network, the smoothed classifier $g$ cannot be evaluated exactly, since it is not possible to exactly compute the expectation of a neural network's prediction over a random input.
However, by repeatedly sampling the random vector $f(T(\mathbf{x}))$, one can obtain upper and lower bounds on the expected value of each entry of that vector, which hold with high probability over the sampling procedure.
In the hard smoothing case, since each entry of $f(T(\mathbf{x}))$ is a Bernoulli random variable, one can use standard Bernoulli confidence intervals like the Clopper-Pearson, as in \cite{lecuyer2018certified, cohen2019certified}.
In the soft smoothing case, since each entry of $f(T(\mathbf{x}))$  is bounded in $[0, 1]$, one can use Hoeffding-style concentration inequalities to derive high-probability confidence intervals for the entries of $f(T(\mathbf{x}))$ .

\paragraph{Gaussian smoothing}
When $T$ is an additive Gaussian corruption,
$$ T(\mathbf{x}) = \mathbf{x} + \boldsymbol{\varepsilon}, \quad \boldsymbol{\varepsilon} \sim \mathcal{N}(0, \sigma^2 I), $$
the robust classifier $g: \mathbb{R}^d \to [k]$ is given by:
\begin{align}
g(\mathbf{x}) = \; \argmax_{j \in [k]} \; \hat{f}_\sigma(\mathbf{x}) \quad \text{where} \quad \hat{f}_\sigma(\mathbf{x}) = \mathbb{E}_{\boldsymbol{\varepsilon} \sim \mathcal{N}(0, \sigma^2 I)} [f(\mathbf{x} + \boldsymbol{\varepsilon})].
\label{eq:smoothing-defn}
\end{align}
Gaussian-smoothed classifiers are \emph{certifiably} robust under the $\ell_2$ norm: for any input $\mathbf{x}$, if we know $\hat{f}_\sigma(\mathbf{x})$, we can certify that $g$'s prediction will remain constant within an $\ell_2$ ball around $\mathbf{x}$:

\begin{thm}[Extension to ``soft smoothing'' of Theorem 1 from \cite{cohen2019certified}; see also Appendix A in \cite{salman2019provably}]
Let $f: \mathbb{R}^d \to \Delta_k$ be any function, and define $g$ and $\hat{f}_\sigma$ as in \eqref{eq:smoothing-defn}.
For some $\mathbf{x} \in \mathbb{R}^d$, let $y_1, y_2 \in [k]$ be the indices of the largest and second-largest entries of $\hat{f}_\sigma(\mathbf{x})$.  Then $g(\mathbf{x} + \boldsymbol{\delta}) = y_1$ for any $\boldsymbol{\delta}$ with
$$ \|\boldsymbol{\delta}\|_2 \le \frac{\sigma}{2} \left( \Phi^{-1}(\hat{f}_\sigma(\mathbf{x})_{y_1}) - \Phi^{-1}(\hat{f}_\sigma(\mathbf{x})_{y_2}) \right). $$
\label{thm:main-theorem}
\end{thm}

Theorem \ref{thm:main-theorem} is easy to prove using the following mathematical fact:

\begin{lemma}[Lemma 2 from \cite{salman2019provably}, Lemma 1 from \cite{levine2019certifiably}]
Let $h: \mathbb{R}^d \to [0, 1]$ be any function, and define its Gaussian convolution $\hat{h}_\sigma$ as
    $\hat{h}_\sigma(\mathbf{x}) = \mathbb{E}_{\boldsymbol{\varepsilon} \sim \mathcal{N}(0, \sigma^2 I)}[h(\mathbf{x} + \boldsymbol{\varepsilon})]$.
Then, for any input $\mathbf{x} \in \mathbb{R}^d$ and any perturbation $\boldsymbol{\delta} \in \mathbb{R}^d$,
\begin{align*}
   \Phi \left( \Phi^{-1}(\hat{h}_\sigma(\mathbf{x})) - \frac{\|\boldsymbol \delta\|_2}{\sigma}  \right) \le \hat{h}_\sigma(\mathbf{x} + \boldsymbol{\delta}) \le \Phi \left(\Phi^{-1}(\hat{h}_\sigma(\mathbf{x})) + \frac{\|\boldsymbol \delta\|_2}{\sigma}  \right).
\end{align*}
\label{lemma:convolution-lipschitz}
\end{lemma}

Intuitively, Lemma \ref{lemma:convolution-lipschitz} says that $\hat{h}_\sigma(\mathbf{x} + \boldsymbol{\delta})$ cannot be too much larger or too much smaller than $\hat{h}_\sigma(\mathbf{x})$.
If this has the feel of a Lipschitz guarantee, there is good reason: Lemma \ref{lemma:convolution-lipschitz} is equivalent to the statement that the function $\mathbf{x} \mapsto \Phi^{-1}(\hat{h}_\sigma(\mathbf{x}))$ is $1/\sigma$-Lipschitz.

Theorem \ref{thm:main-theorem} is a direct consequence of Lemma \ref{lemma:convolution-lipschitz}:

\begin{proof}[Proof of Theorem \ref{thm:main-theorem}]

Since the outputs of $\hat{f}_\sigma$ live in the probability simplex, for each class $j$ the function $\hat{f}_\sigma(\cdot)_j$ has output bounded in $[0, 1]$, and hence can be viewed as a function $\hat{h}_\sigma$ for which the condition of Lemma \ref{lemma:convolution-lipschitz} applies.

Therefore, from applying Lemma \ref{lemma:convolution-lipschitz} to $\hat{f}_\sigma(\cdot)_{y_1}$, we know that:
$$ \hat{f}_\sigma(\mathbf{x} + \boldsymbol{\delta})_{y_1} \ge \Phi \left(\Phi^{-1}(\hat{f}_\sigma(\mathbf{x})_{y_1}) - \frac{\|\boldsymbol \delta\|_2}{\sigma} \right) $$
and, for any $j$, from applying Lemma \ref{lemma:convolution-lipschitz} to $\hat{f}_\sigma(\cdot)_{j}$, we know that:
$$ \Phi \left(\Phi^{-1}(\hat{f}_\sigma(\mathbf{x})_j) + \frac{\|\boldsymbol \delta\|_2}{\sigma} \right) \ge \hat{f}_\sigma(\mathbf{x} + \boldsymbol{\delta})_j. $$
Combining these two results, it follows that a sufficient condition for $\hat{f}_\sigma(\mathbf{x} + \boldsymbol{\delta})_{y_1} \ge \hat{f}_\sigma(\mathbf{x} + \boldsymbol{\delta})_j$ is:
$$  \Phi \left(\Phi^{-1}(\hat{f}_\sigma(\mathbf{x})_{y_1}) - \frac{\|\boldsymbol \delta\|_2}{\sigma} \right) \ge  \Phi \left(\Phi^{-1}(\hat{f}_\sigma(\mathbf{x})_j) + \frac{\|\boldsymbol \delta\|_2}{\sigma} \right), $$
or equivalently,
$$ \|\boldsymbol{\delta}\|_2 \le \frac{\sigma}{2}( \Phi^{-1}(\hat{f}_\sigma(\mathbf{x})_{y_1} - \Phi^{-1}(\hat{f}_\sigma (\mathbf{x})_j) ). $$
Hence, we can conclude that $\hat{f}_\sigma(\mathbf{x} + \boldsymbol{\delta})_{y_1} \ge \max_{j \neq y_1} \hat{f}_\sigma(\mathbf{x} + \boldsymbol{\delta})_j$ so long as 
$$ \|\boldsymbol{\delta}\|_2 \le \min_{j \neq y_1} \left \{ \frac{\sigma}{2}( \Phi^{-1}(\hat{f}_\sigma(\mathbf{x})_{y_1} - \Phi^{-1}(\hat{f}_\sigma (\mathbf{x})_j) ) \right \} =  \frac{\sigma}{2}( \Phi^{-1}(\hat{f}_\sigma(\mathbf{x})_{y_1} - \Phi^{-1}(\hat{f}_\sigma (\mathbf{x})_{y_2}) ) $$

\end{proof}

\paragraph{Training}

Given a dataset, a base classifier architecture, and a smoothing level $\sigma > 0$, it currently an active research question to figure out the best way to train the base classifier $f$ so that the smoothed classifier $g$ will attain high certified or empirical robust accuracies.
The original randomized smoothing paper \cite{lecuyer2018certified} proposed training $f$ with Gaussian data augmentation and the standard cross-entropy loss.
However, \cite{salman2019provably} and \cite{li2018second, carmon2019unlabeled} showed that alternative training schemes yield substantial gains in certified accuracy.
In particular, \cite{salman2019provably} proposed training $f$ by performing adversarial training on $g$, and \cite{li2018second, carmon2019unlabeled} proposed training $f$ via stability training \cite{zheng2016improving}.

\paragraph{Related work}
Gaussian smoothing was first proposed as a certified adversarial defense by \cite{lecuyer2018certified} under the name ``PixelDP,'' though similar techniques had been proposed earlier as a heuristic defenses in \cite{cao2017mitigating, liu2018towards}.
Subsequently, \cite{li2018second} proved a stronger robustness guarantee, and finally \cite{cohen2019certified} derived the tightest possible robustness guarantee in the ``hard smooothing'' case, which was extended to the ``soft smoothing'' case by \cite{levine2019certifiably, salman2019provably}.

Concurrently, \cite{zhang2019discretization} proved a robustness guarantee in $\ell_\infty$ norm for Gaussian smoothing; however, since Gaussian smoothing specifically confers $\ell_2$ (not $\ell_\infty$) robustness \cite{cohen2019certified}, the certified accuracy numbers reported in \cite{zhang2019discretization} were weak.

\cite{pinot2019theoretical} gave theoretical and empirical arguments for an adversarial defense similar to randomized smoothing, but did not position their method as a certified defense.

\cite{lee2019stratified} have extended randomized smoothing beyond Gaussian noise / $\ell_2$ norm by proposing a randomization scheme which allows for certified robustness in the $\ell_0$ norm.

%% file: appendix-objective-functions.tex
This appendix details the procedure used to generate the images that appeared in this paper. 

As in \cite{santurkar2019image}, to generate an image $\mathbf{x}^* \in \mathbb{R}^d$ near the starting image $\mathbf{x}_0$ that is classified by a smoothed neural network $\hat{f}_\sigma$ as some target class $t$, we use projected steepest descent to solve the optimization problem:
\begin{align}
    \mathbf{x}^* = \argmin_{\mathbf{x}: \; \| \mathbf{x} - \mathbf{x}_0\|_2 \le \epsilon} \; L(\hat{f}_\sigma, \mathbf{x}, t)
    \label{eq:problem}
\end{align}
where $L$ is a loss function measuring the extent to which $\hat{f}_\sigma$ classifies $\mathbf{x}$ as class $t$.

The two big choices which need to be made are: which loss function to use, and how to compute its gradient?

\paragraph{Loss functions for adversarially-trained networks}
We first review two loss functions for generating images using adversarially-trained neural networks.
Our loss functions for smoothed neural networks (presented below) are inspired by these.

The first is the cross-entropy loss.  If $f^{\text{adv}}: \mathbb{R}^d \to \Delta_k$ is an (adversarially trained) neural network classifier that ends in a softmax layer (so that its output lies on the probability simplex $\Delta_k$), the cross-entropy loss is defined as:
\begin{align*}
    L_{\text{CE}}(f^{\text{adv}}, \mathbf{x}, t) := - \log f^{\text{adv}}(\mathbf{x})_t.
\end{align*}

The second is the ``target class max'' (TCM) loss.
If we write $f^{\text{adv}}$ as $f^{\text{adv}}(\mathbf{x}) = \softmax(\logits(\mathbf{x}))$, where $\logits: \mathbb{R}^d \to \mathbb{R}^k$ is $f^{\text{adv}}$ minus the final softmax layer, then 
the TCM loss is defined as:
\begin{align*}
    L_{\text{TCM}}(f^{\text{adv}}, \mathbf{x}, t) := -\logits(\mathbf{x})_t.
\end{align*}
In other words, minimizing $L_{\text{TCM}}$ will maximize the score that $\logits$ assigns to class $t$.

Since $f^{\text{adv}}$ is just a neural network, computing the gradients of these loss functions can be easily done using automatic differentiation.  (The situation is more complicated for smoothed neural networks.)

We note that \cite{santurkar2019image} used $L_{\text{CE}}$ in their experiments.


\paragraph{Loss functions for smoothed networks}
Our loss functions for smoothed neural networks are inspired by those described above for adversarially trained networks.
If $\hat{f}_\sigma$ is a smoothed neural network of the form $\hat{f}_\sigma(\mathbf{x}) = \mathbb{E}_{\boldsymbol \varepsilon \sim \mathcal{N}(0, \sigma^2 I)}[ f(\mathbf x + \boldsymbol \varepsilon) ]$, with $f$ a neural network that ends in a softmax layer, then the cross-entropy loss is defined as:
\begin{align}
    L_{\text{CE}}(\hat{f}_\sigma, \mathbf{x}, t) := - \log \hat{f}_\sigma(\mathbf{x})_t = - \log \mathbb{E}_{ \boldsymbol \varepsilon \sim \mathcal{N}(0, \sigma^2 I)} [ f(\mathbf{x} + \boldsymbol{\varepsilon})_t ].
    \label{eq:cross-entropy}
\end{align}

If we decompose $f$ as $f(\mathbf{x}) = \softmax(\logits(\mathbf{x}))$, where $\logits: \mathbb{R}^d \to \mathbb{R}^k$ is $f$ minus the softmax layer, then the TCM loss is defined as:
\begin{align}
    L_{\text{TCM}}(\hat{f}_\sigma, \mathbf{x}, t) := - \mathbb{E}_{ \boldsymbol \varepsilon \sim \mathcal{N}(0, \sigma^2 I)} [ \logits(\mathbf{x} + \boldsymbol{\varepsilon})_t ].
    \label{eq:tcm}
\end{align}
In other words, minimizing $ L_{\text{TCM}}$ will maximize the expected logit of class $t$ for the random input $\mathbf{x} + \boldsymbol{\varepsilon}$

\paragraph{Gradient estimators}

To solve problem \eqref{eq:problem} using PGD, we need to be able to compute the gradient of the objective w.r.t $\mathbf{x}$.
However, for smoothed neural networks, it is not possible to exactly compute the gradient of either $L_{\text{CE}}$  or $L_{\text{TCM}}$.
We therefore must resort to gradient estimates obtained using Monte Carlo sampling.

For $L_{\text{TCM}}$, we use the following unbiased gradient estimator:
\begin{align*}
    \nabla_{\mathbf{x}} L_{\text{TCM}}(\hat{f}_\sigma, \mathbf{x}, t) \approx - \frac{1}{N} \sum_{i=1}^N \nabla_{\mathbf{x}} \logits(\mathbf{x} + \boldsymbol{\varepsilon}_i)_t, \quad \boldsymbol{\varepsilon}_i \sim \mathcal{N}(0, \sigma^2 I)
\end{align*}

This estimator is unbiased since
\begin{align*}
    \mathbb{E}_{\boldsymbol{\varepsilon}_1, \hdots, \boldsymbol{\varepsilon}_N \sim \mathcal{N}(0, \sigma^2 I)} \left[  - \frac{1}{N} \sum_{i=1}^N \nabla_{\mathbf{x}} \logits(\mathbf{x} + \boldsymbol{\varepsilon}_i)_t \right]
    &= \mathbb{E}_{\boldsymbol{\varepsilon} \sim \mathcal{N}(0, \sigma^2 I)} \left[  - \nabla_{\mathbf{x}} \logits(\mathbf{x} + \boldsymbol{\varepsilon})_t \right] \\
    &=  \nabla_{\mathbf{x}} \; \mathbb{E}_{\boldsymbol{\varepsilon} \sim \mathcal{N}(0, \sigma^2 I)} \left[  - \logits(\mathbf{x} + \boldsymbol{\varepsilon})_t \right].
\end{align*}

For $L_{\text{CE}}$, we are unaware of any unbiased gradient estimator, so, following \cite{salman2019provably}, we use the following biased ``plug-in'' gradient estimator:

\begin{align*}
    \nabla_{\mathbf{x}} L_{\text{CE}}(\hat{f}_\sigma, \mathbf{x}, t) \approx \nabla_{\mathbf{x}} \left[  - \log \left( \frac{1}{N} \sum_{i=1}^N  f(\mathbf{x} + \boldsymbol{\varepsilon}_i)_t \right) \right], \quad \boldsymbol{\varepsilon}_i \sim \mathcal{N}(0, \sigma^2 I)
\end{align*}

\paragraph{Experimental comparison between loss functions}

Figure \ref{fig:three-objectives} shows large-$\epsilon$ adversarial examples crafted for a smoothed neural network using both $L_{\text{TCM}}$ and $L_{\text{CE}}$.
The adversarial examples crafted using $L_{\text{TCM}}$  seem to better perceptually resemble the target class.
Therefore, in this work we primarily use $L_{\text{TCM}}$.

\paragraph{Experimental comparison between training procedures}

For most of the figures in this paper, we used a base classifier from \cite{cohen2019certified} trained using Gaussian data augmentation.
However, in Figures \ref{fig:adv-smooth-pt1}-\ref{fig:adv-smooth-pt3}, we compare large-$\epsilon$ adversarial examples for this base classifier to those synthesized for a base classifier trained using the \textsc{SmoothAdv} procedure from \cite{salman2019provably}, which was shown in that paper to attain much better certified accuracies than the network from \cite{cohen2019certified}.
We find that there does not seem to be a large difference in the perceptual quality of the generated images.
Therefore, throughout this paper we used the network from \cite{cohen2019certified}, since we wanted to emphasize that perceptually-aligned gradients arise even with robust classifiers that do not involve adversarial training of any kind.

\paragraph{Experimental study of number of Monte Carlo samples}

One important question is how many Monte Carlo samples $N$ are needed when computing the gradient of $L_{\text{TCM}}$ or $L_{\text{CE}}$.
In Figure \ref{fig:noise-sample-variation} we show large-$\epsilon$ adversarial examples synthesized using $N \in \{1, 5, 20, 25, 50, 75\}$ Monte Carlo samples. 
There does not seem to be a large difference between using $N=20$ samples or using more than 20.
Images synthesized using $N=1$ samples do appear a bit less developed than the others (e.g. the terrier with $N=1$ is has fewer ears than when $N$ is large.)
In this work, we primarily used $N=20$.

\paragraph{Hyperparameters} 
The following table shows the hyperparameter settings for all of the figures in this paper.

\begin{center}
\begin{tabular}{ r | c c c c c}
 Figure & $\sigma$ & number of PGD steps & $\epsilon$ & PGD step size & $N$ \\ 
 \hline
 \ref{fig:three-models-compared}, \ref{fig:three-models-compared-all} & 0.5 & 300 & 40.0 & 2.8 (vanilla), 0.7 & 20 \\ 
 \ref{fig:large-epsilon} & 0.5 & 300 & 40.0 & 0.7 & 20 \\ 
 \ref{fig:generation6} & 0.5 & 300 & 40.0 & 0.7 & 20 \\ 
 \ref{fig:noise-level}, \ref{fig:noise-level-variation-pt1}-\ref{fig:noise-level-variation-pt3}  & vary & 300 & 40.0 & 2.8 ($\sigma$ = 0), 0.7 & 20 \\ 
 \ref{fig:adv-smooth-pt1}-\ref{fig:adv-smooth-pt3}  & 0.5, 1.0 & 300 & 40.0 & 0.7 & 20 \\ 
 \ref{fig:three-objectives}  & 0.5 & 300 & 40.0 & 2.0 (CE), 0.7 & 20 \\ 
 \ref{fig:noise-sample-variation}  & 0.5 & 300 & 40.0 & 0.7 & vary \\ 
\end{tabular}
\end{center}

Note that Figures \ref{fig:three-models-compared} and \ref{fig:three-models-compared-all} only use \textit{stepSize} = 2.8 in the Vanilla column, Figure \ref{fig:three-objectives} only uses \textit{stepSize} = 2.0 in the C-E Loss column, and Figures \ref{fig:noise-level} and \ref{fig:noise-level-variation-pt1}-\ref{fig:noise-level-variation-pt3} only use \textit{stepSize} = 2.8 for $\sigma$ = 0.

\newpage

\begin{figure}[h!]
\centering
  \includegraphics[scale=0.5]{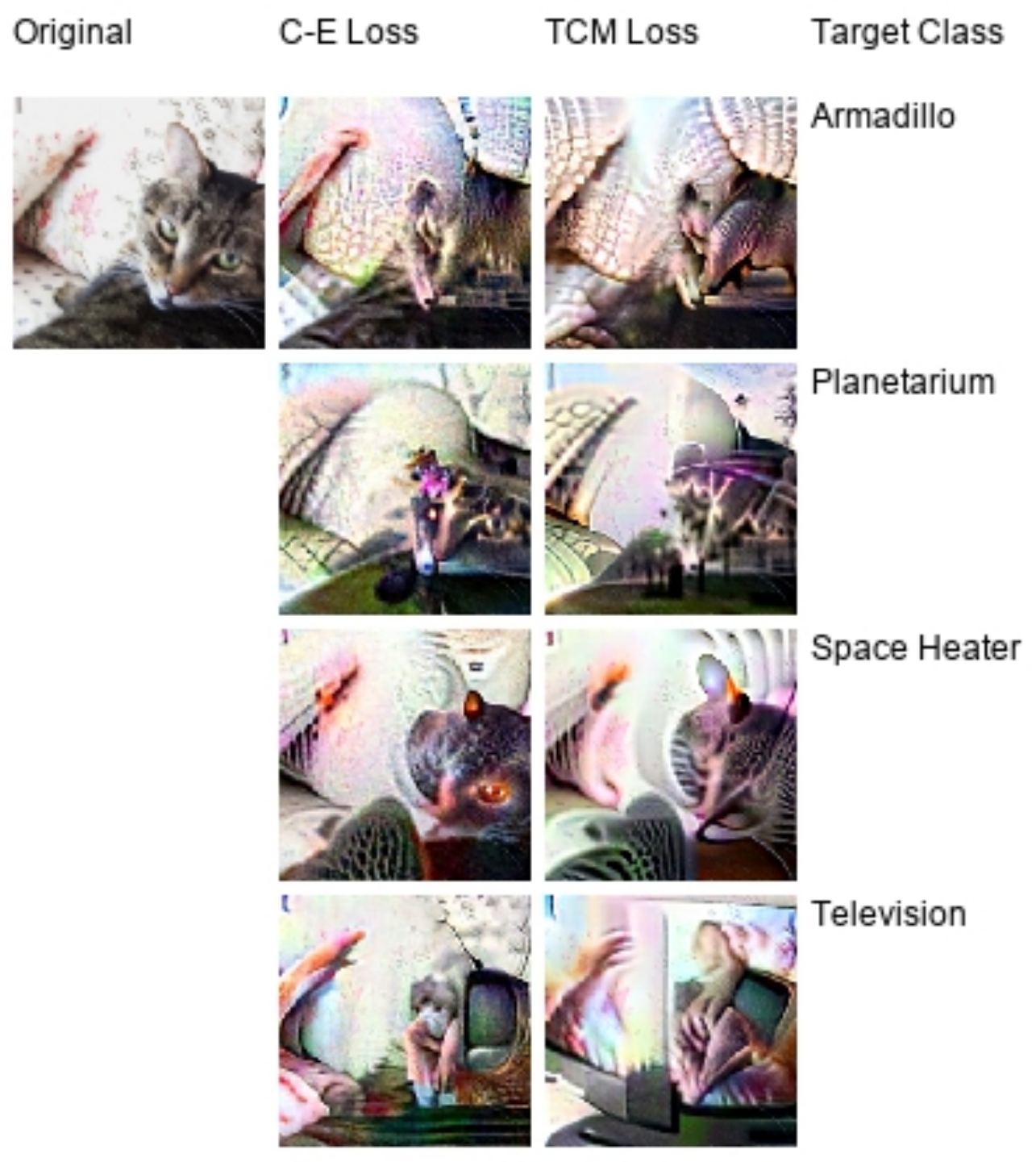} \quad
  \includegraphics[scale=0.5]{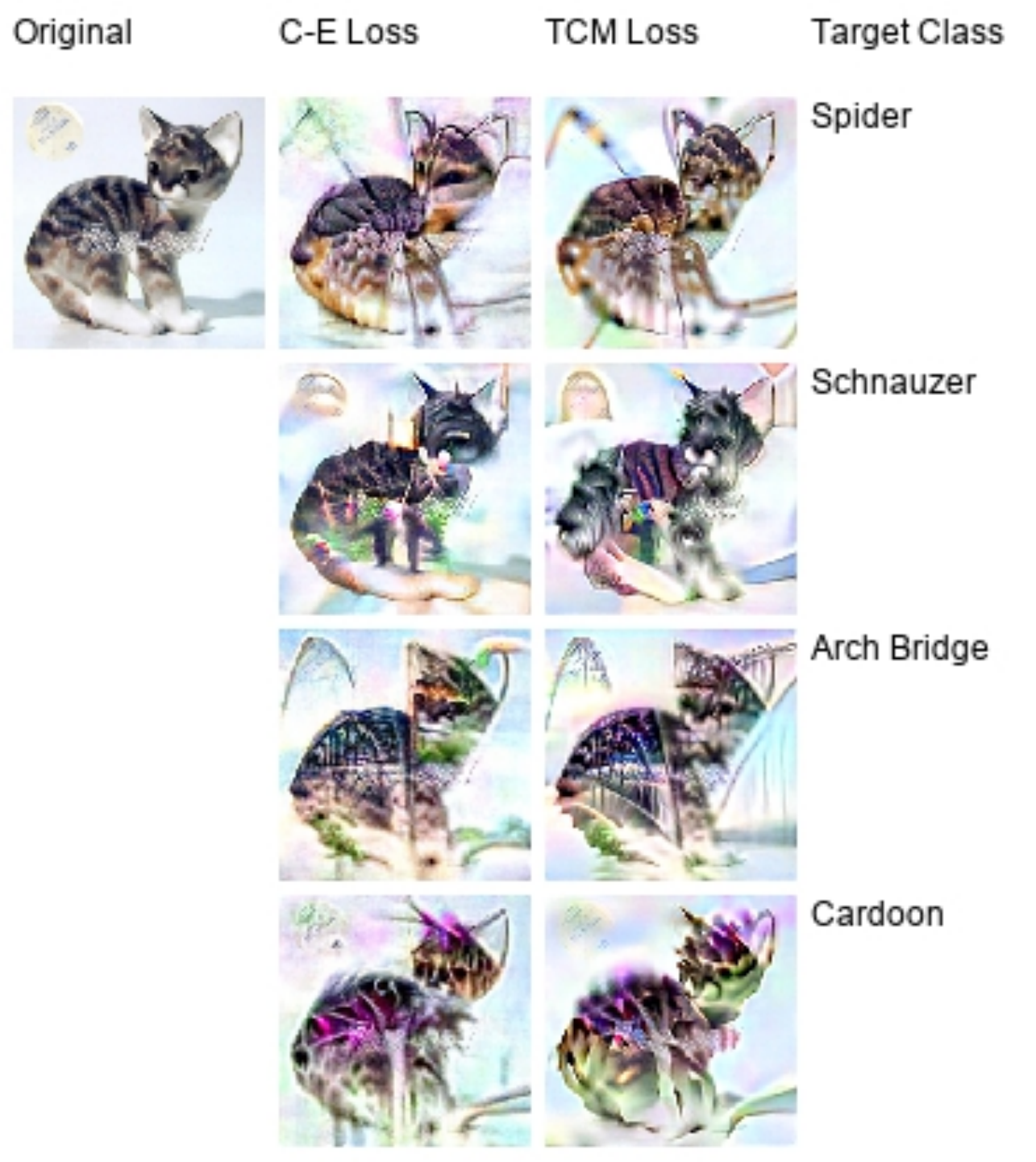} \quad
  \includegraphics[scale=0.5]{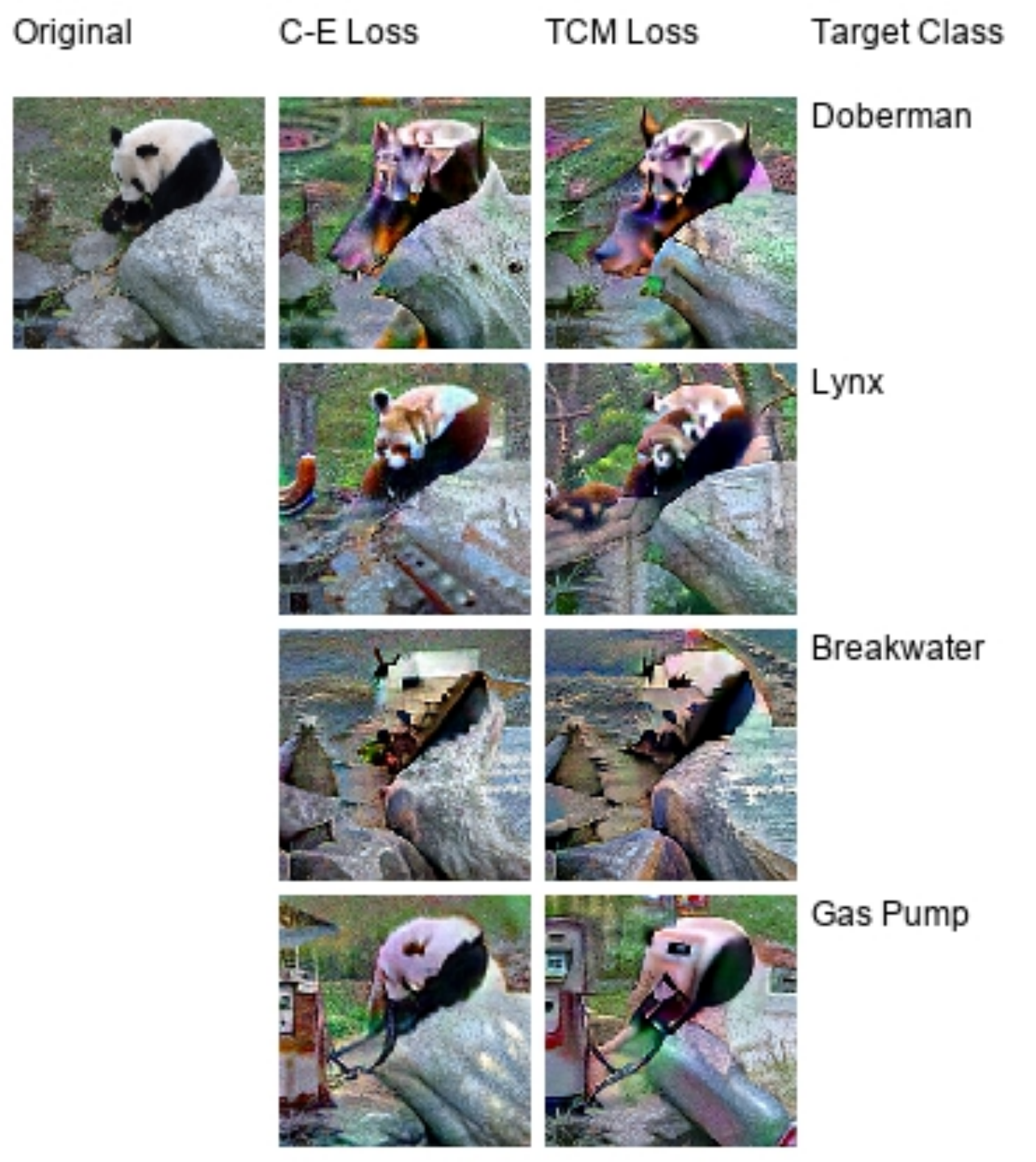} \quad
  \includegraphics[scale=0.5]{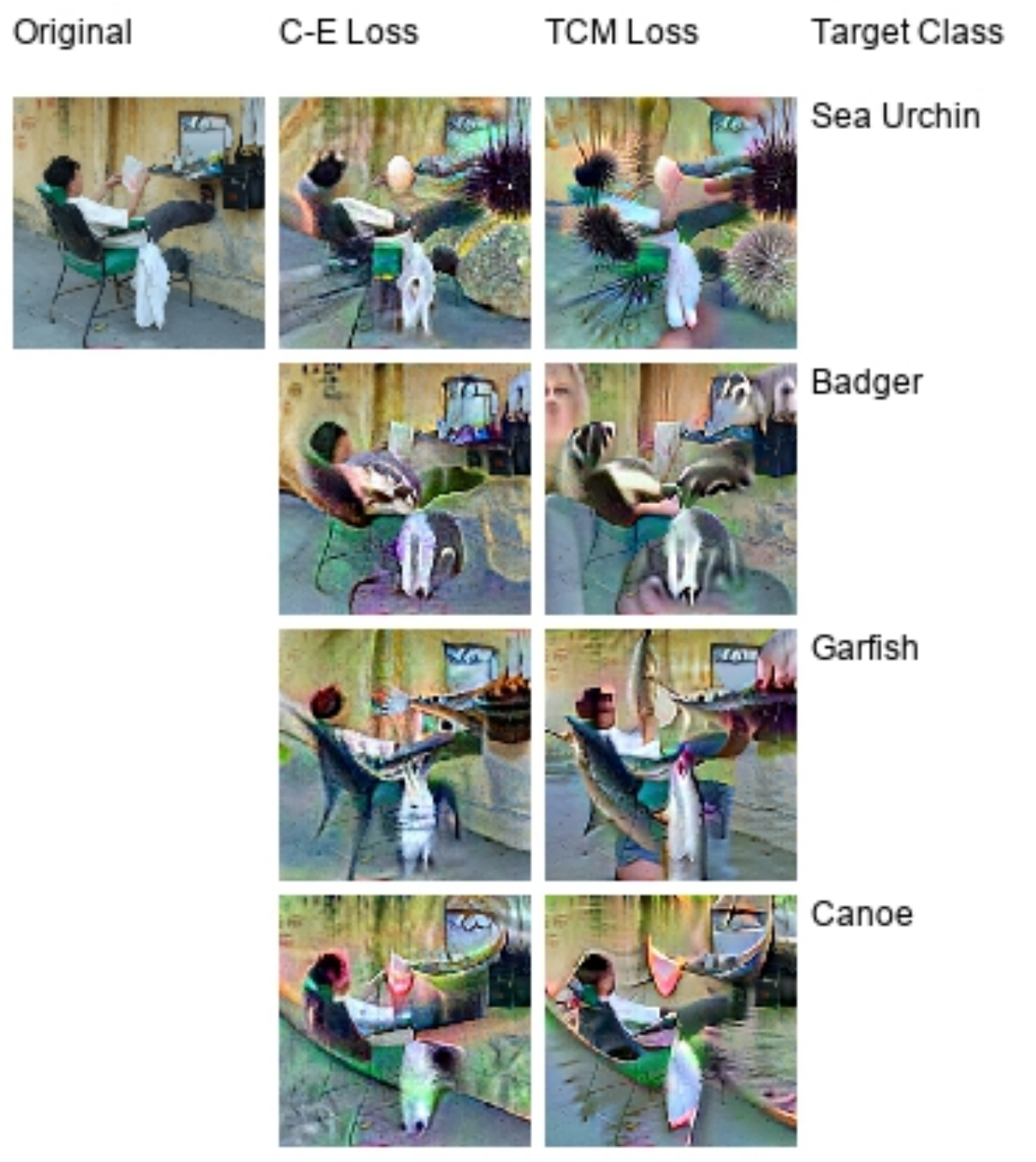} \quad
  \caption{Here, we compare the perceptual quality of large-$\epsilon$ adversarial examples (for a smoothed neural network) crafted using the cross-entropy loss $L_{\text{CE}}$ to those crafted using the target class max $L_{\text{TCM}}$ loss.  Observe that adversarial examples crafted using the TCM loss seem to better resemble the targeted class.  For this reason, we used the TCM loss function throughout this paper. }
  \label{fig:three-objectives}
\end{figure}

\newpage

\begin{figure}[h!]
\centering
  \includegraphics[scale=0.45]{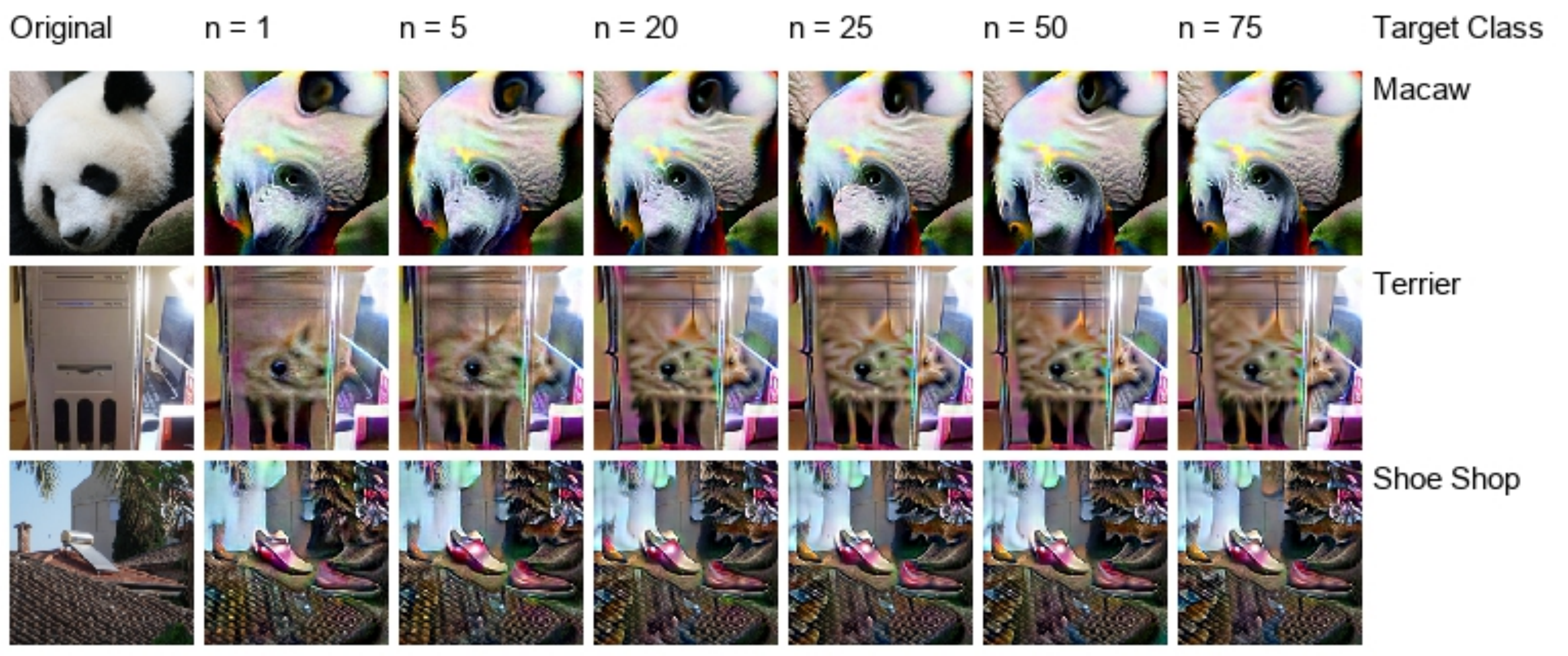}
  \caption{Large-$\epsilon$ adversarial examples for a smoothed neural network crafted using different settings of the parameter $N$, the number of Monte Carlo samples used for gradient estimation.}
  \label{fig:noise-sample-variation}
\end{figure}

\newpage

\begin{figure}[!h]
  \centering
  \includegraphics[scale=0.5]{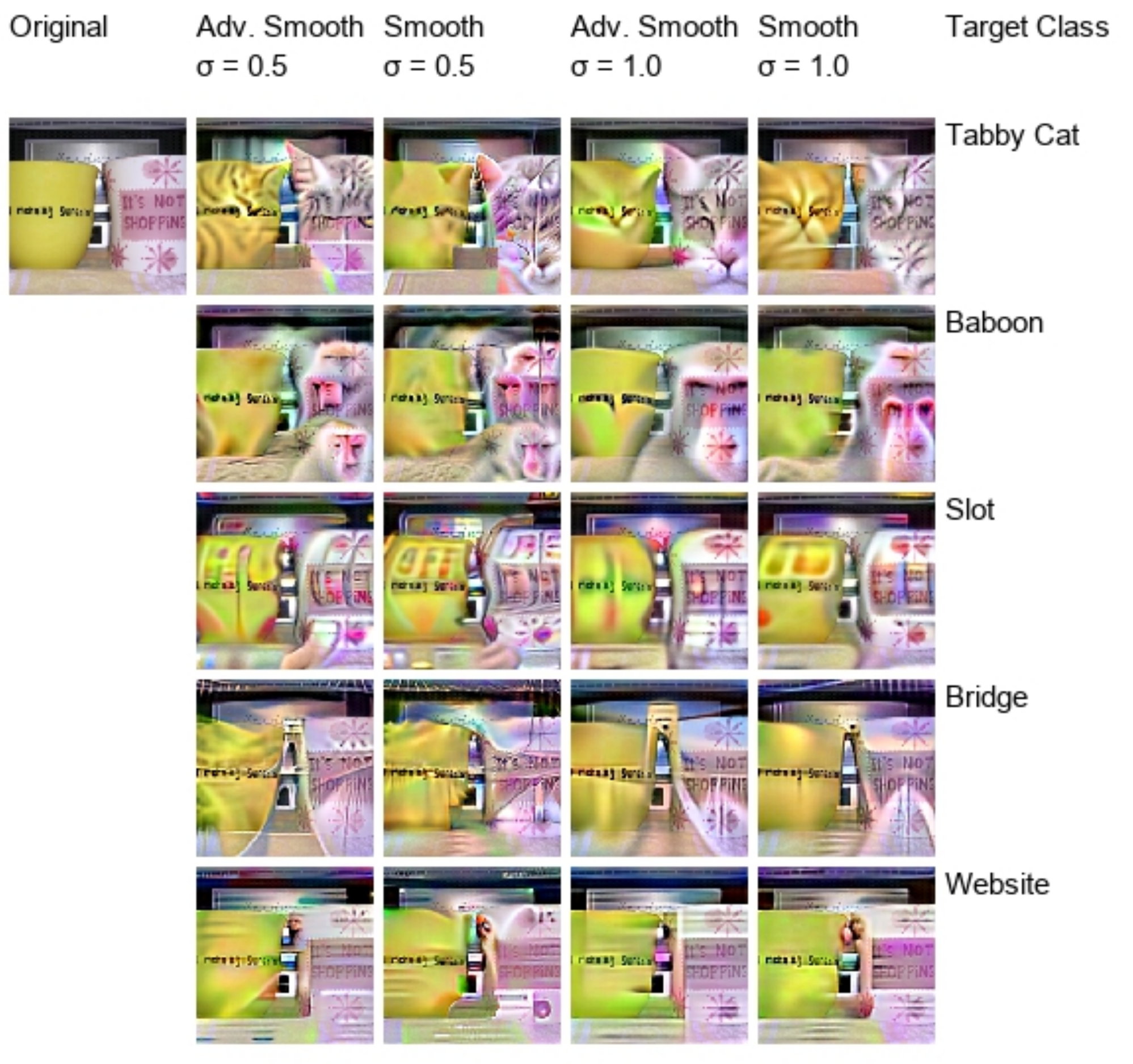}
  \caption{We compare (part 1/3) large-$\epsilon$ targeted adversarial examples for smoothed networks trained using Gaussian data augmentation \cite{lecuyer2018certified, cohen2019certified} (columns ``Smooth'') to those for smoothed networks trained using the \textsc{SmoothAdv} algorithm of \protect\cite{salman2019provably}, i.e. adversarial training on the smoothed classifier  (columns ``Adv. Smooth'').}
  \label{fig:adv-smooth-pt1}
\end{figure}

\newpage

\begin{figure}[!h]
  \centering
  \includegraphics[scale=0.5]{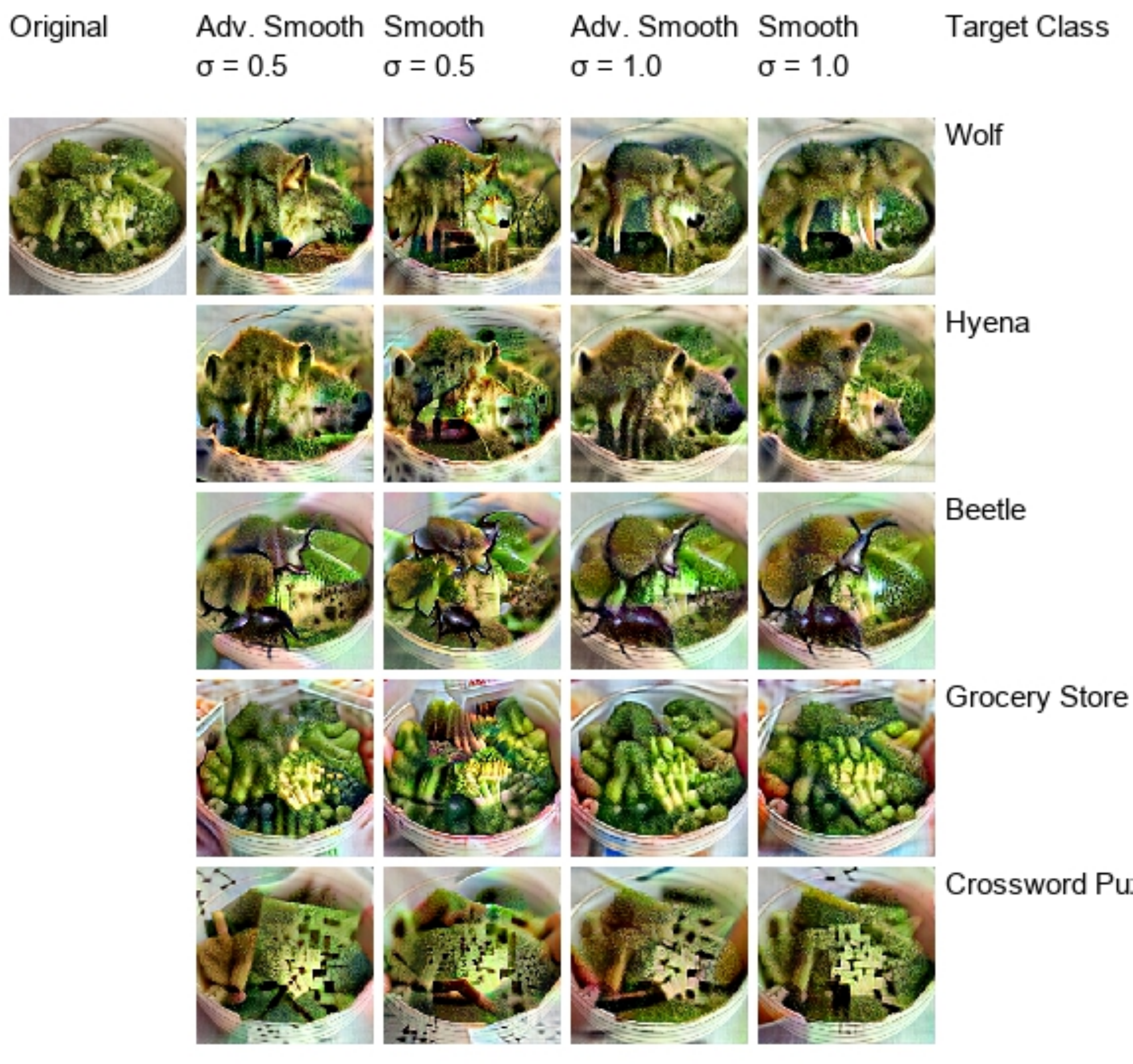}
  \caption{We compare (part 2/3) large-$\epsilon$ targeted adversarial examples for smoothed networks trained using Gaussian data augmentation \cite{lecuyer2018certified, cohen2019certified} (columns ``Smooth'') to those for smoothed networks trained using the \textsc{SmoothAdv} algorithm of \protect\cite{salman2019provably}, i.e. adversarial training on the smoothed classifier  (columns ``Adv. Smooth'').}
  \label{fig:adv-smooth-pt2}
\end{figure}

\newpage

\begin{figure}[!h]
  \centering
  \includegraphics[scale=0.5]{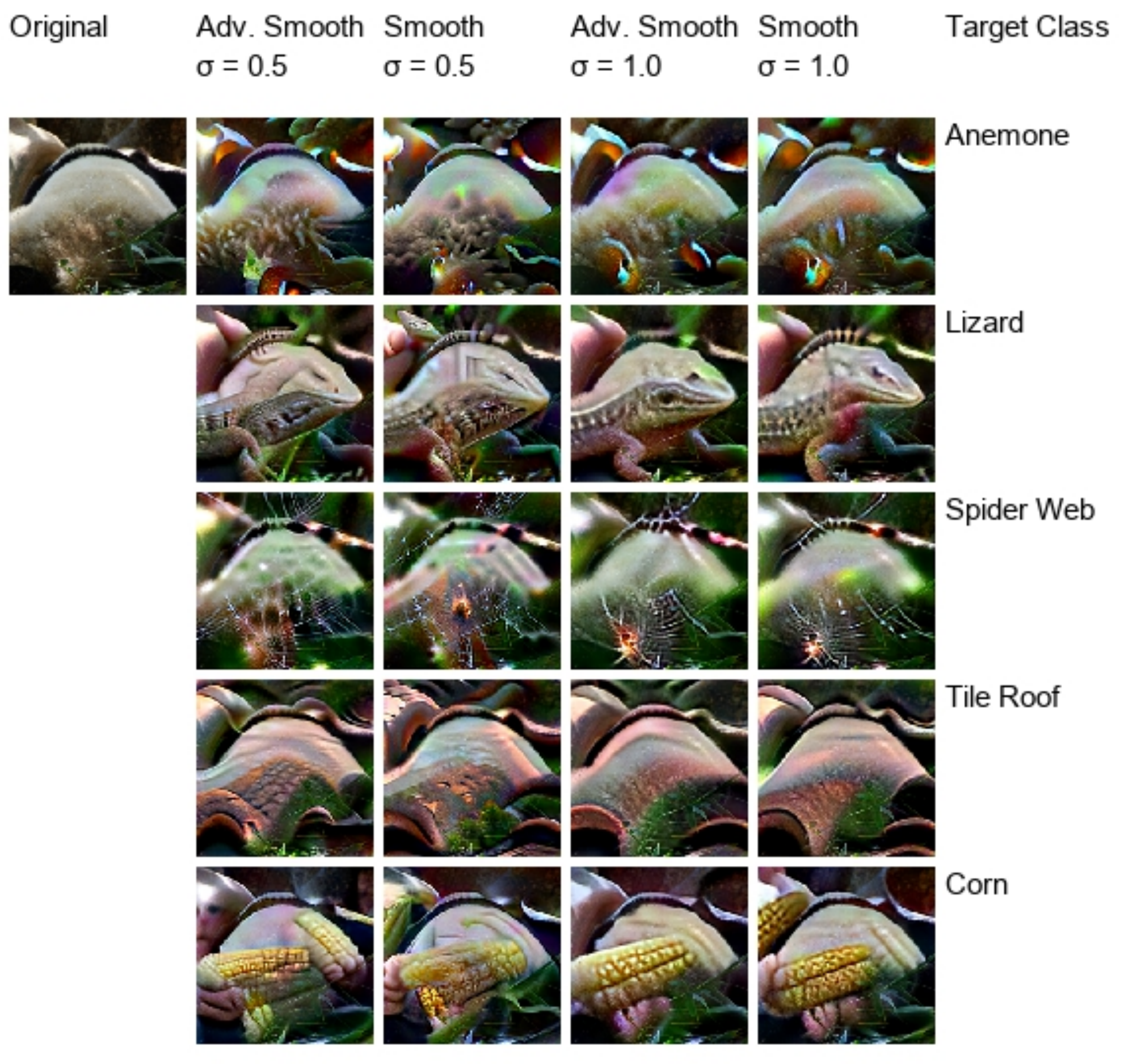}
  \caption{We compare (part 3/3) large-$\epsilon$ targeted adversarial examples for smoothed networks trained using Gaussian data augmentation \cite{lecuyer2018certified, cohen2019certified} (columns ``Smooth'') to those for smoothed networks trained using the \textsc{SmoothAdv} algorithm of \protect\cite{salman2019provably}, i.e. adversarial training on the smoothed classifier  (columns ``Adv. Smooth'').}
  \label{fig:adv-smooth-pt3}
\end{figure}

%% file: main.bbl
\begin{thebibliography}{28}
\providecommand{\natexlab}[1]{#1}
\providecommand{\url}[1]{\texttt{#1}}
\expandafter\ifx\csname urlstyle\endcsname\relax
  \providecommand{\doi}[1]{doi: #1}\else
  \providecommand{\doi}{doi: \begingroup \urlstyle{rm}\Url}\fi

\bibitem[Szegedy et~al.(2014)Szegedy, Zaremba, Sutskever, Bruna, Erhan,
  Goodfellow, and Fergus]{szegedy2014intriguing}
Christian Szegedy, Wojciech Zaremba, Ilya Sutskever, Joan Bruna, Dumitru Erhan,
  Ian Goodfellow, and Rob Fergus.
\newblock Intriguing properties of neural networks.
\newblock In \emph{International Conference on Learning Representations}, 2014.

\bibitem[Biggio et~al.(2013)Biggio, Corona, Maiorca, Nelson, Šrndić, Laskov,
  Giacinto, and Roli]{biggio2013evasion}
Battista Biggio, Igino Corona, Davide Maiorca, Blaine Nelson, Nedim Šrndić,
  Pavel Laskov, Giorgio Giacinto, and Fabio Roli.
\newblock Evasion attacks against machine learning at test time.
\newblock \emph{Joint European Conference on Machine Learning and Knowledge
  Discovery in Database}, 2013.

\bibitem[Kurakin et~al.(2016)Kurakin, Goodfellow, and
  Bengio]{kurakin2017adversarial}
Alexey Kurakin, Ian Goodfellow, and Samy Bengio.
\newblock Adversarial machine learning at scale.
\newblock \emph{arXiv preprint arXiv:1611.01236}, 2016.

\bibitem[Madry et~al.(2018)Madry, Makelov, Schmidt, Tsipras, and
  Vladu]{madry2017towards}
Aleksander Madry, Aleksandar Makelov, Ludwig Schmidt, Dimitris Tsipras, and
  Adrian Vladu.
\newblock Towards deep learning models resistant to adversarial attacks.
\newblock In \emph{International Conference on Learning Representations}, 2018.

\bibitem[Lecuyer et~al.(2019)Lecuyer, Atlidakis, Geambasu, Hsu, and
  Jana]{lecuyer2018certified}
M.~Lecuyer, V.~Atlidakis, R.~Geambasu, D.~Hsu, and S.~Jana.
\newblock Certified robustness to adversarial examples with differential
  privacy.
\newblock In \emph{IEEE Symposium on Security and Privacy (SP)}, 2019.

\bibitem[Li et~al.(2019)Li, Chen, Wang, and Carin]{li2018second}
Bai Li, Changyou Chen, Wenlin Wang, and Lawrence Carin.
\newblock Certified adversarial robustness with additive gaussian noise.
\newblock In \emph{Advances in Neural Information Processing Systems
  (NeurIPS)}, 2019.

\bibitem[Cohen et~al.(2019)Cohen, Rosenfeld, and Kolter]{cohen2019certified}
Jeremy Cohen, Elan Rosenfeld, and Zico Kolter.
\newblock Certified adversarial robustness via randomized smoothing.
\newblock In \emph{Proceedings of the 36th International Conference on Machine
  Learning}, 2019.

\bibitem[Salman et~al.(2019{\natexlab{a}})Salman, Yang, Li, Zhang, Zhang,
  Razenshteyn, and Bubeck]{salman2019provably}
Hadi Salman, Greg Yang, Jerry Li, Pengchuan Zhang, Huan Zhang, Ilya
  Razenshteyn, and Sebastien Bubeck.
\newblock Provably robust deep learning via adversarially trained smoothed
  classifiers.
\newblock In \emph{Advances in Neural Information Processing Systems
  (NeurIPS)}, 2019{\natexlab{a}}.

\bibitem[Tsipras et~al.(2019)Tsipras, Santurkar, Engstrom, Turner, and
  Madry]{tsipras2018robustness}
Dimitris Tsipras, Shibani Santurkar, Logan Engstrom, Alexander Turner, and
  Aleksander Madry.
\newblock Robustness may be at odds with accuracy.
\newblock In \emph{International Conference on Learning Representations}, 2019.
\newblock URL \url{https://openreview.net/forum?id=SyxAb30cY7}.

\bibitem[Santurkar et~al.(2019)Santurkar, Tsipras, Tran, Ilyas, Engstrom, and
  Madry]{santurkar2019image}
Shibani Santurkar, Dimitris Tsipras, Brandon Tran, Andrew Ilyas, Logan
  Engstrom, and Aleksander Madry.
\newblock Image synthesis with a single (robust) classifier.
\newblock In \emph{Advances in Neural Information Processing Systems
  (NeurIPS)}, 2019.

\bibitem[Engstrom et~al.(2019)Engstrom, Ilyas, Santurkar, Tsipras, Tran, and
  Madry]{engstrom2019learning}
Logan Engstrom, Andrew Ilyas, Shibani Santurkar, Dimitris Tsipras, Brandon
  Tran, and Aleksander Madry.
\newblock Adversarial robustness as a prior for learned representations.
\newblock \emph{arXiv preprint arXiv:1906.00945}, 2019.

\bibitem[Olah et~al.(2017)Olah, Mordvintsev, and Schubert]{olah2017feature}
Chris Olah, Alexander Mordvintsev, and Ludwig Schubert.
\newblock Feature visualization.
\newblock \emph{Distill}, 2017.
\newblock \doi{10.23915/distill.00007}.
\newblock https://distill.pub/2017/feature-visualization.

\bibitem[Nguyen et~al.(2015)Nguyen, Yosinski, and Clune]{nguyen2014deep}
Anh~Mai Nguyen, Jason Yosinski, and Jeff Clune.
\newblock Deep neural networks are easily fooled: High confidence predictions
  for unrecognizable images.
\newblock \emph{2015 IEEE Conference on Computer Vision and Pattern Recognition
  (CVPR)}, 2015.

\bibitem[Mahendran and Vedaldi(2015)]{mahendran2015understanding}
Aravindh Mahendran and Andrea Vedaldi.
\newblock Understanding deep image representations by inverting them.
\newblock \emph{2015 IEEE Conference on Computer Vision and Pattern Recognition
  (CVPR)}, Jun 2015.
\newblock \doi{10.1109/cvpr.2015.7299155}.
\newblock URL \url{http://dx.doi.org/10.1109/CVPR.2015.7299155}.

\bibitem[{\O}ygard(2015)]{oygard2015visualizing}
Audun~M. {\O}ygard.
\newblock Visualizing googlenet claasses.
\newblock
  \url{https://www.auduno.com/2015/07/29/visualizing-googlenet-classes/}, 2015.
\newblock [Online; accessed 30-August-2019].

\bibitem[Athalye et~al.(2018)Athalye, Carlini, and
  Wagner]{athalye2018obfuscated}
Anish Athalye, Nicholas Carlini, and David Wagner.
\newblock Obfuscated gradients give a false sense of security: Circumventing
  defenses to adversarial examples.
\newblock In \emph{Proceedings of the 35th International Conference on Machine
  Learning}, 2018.

\bibitem[Brendel et~al.(2019)Brendel, Rauber, K{\"u}mmerer, Ustyuzhaninov, and
  Bethge]{brendel2019accurate}
Wieland Brendel, Jonas Rauber, Matthias K{\"u}mmerer, Ivan Ustyuzhaninov, and
  Matthias Bethge.
\newblock Accurate, reliable and fast robustness evaluation.
\newblock In \emph{Advances in Neural Information Processing Systems
  (NeurIPS)}, 2019.

\bibitem[Salman et~al.(2019{\natexlab{b}})Salman, Yang, Zhang, Hsieh, and
  Zhang]{salman2019convex}
Hadi Salman, Greg Yang, Huan Zhang, Cho-Jui Hsieh, and Pengchuan Zhang.
\newblock A convex relaxation barrier to tight robustness verification of
  neural networks.
\newblock In \emph{Advances in Neural Information Processing Systems
  (NeurIPS)}, 2019{\natexlab{b}}.

\bibitem[Carmon et~al.(2019)Carmon, Raghunathan, Schmidt, Liang, and
  Duchi]{carmon2019unlabeled}
Y.~Carmon, A.~Raghunathan, L.~Schmidt, P.~Liang, and J.~C. Duchi.
\newblock Unlabeled data improves adversarial robustness.
\newblock In \emph{Advances in Neural Information Processing Systems
  (NeurIPS)}, 2019.

\bibitem[Zheng et~al.(2016)Zheng, Song, Leung, and
  Goodfellow]{zheng2016improving}
Stephan Zheng, Yang Song, Thomas Leung, and Ian~J. Goodfellow.
\newblock Improving the robustness of deep neural networks via stability
  training.
\newblock In \emph{Computer Vision and Pattern Recognition}, 2016.

\bibitem[He et~al.(2016)He, Zhang, Ren, and Sun]{he2016deep}
Kaiming He, Xiangyu Zhang, Shaoqing Ren, and Jian Sun.
\newblock Deep residual learning for image recognition.
\newblock \emph{2016 IEEE Conference on Computer Vision and Pattern Recognition
  (CVPR)}, Jun 2016.

\bibitem[Deng et~al.(2009)Deng, Dong, Socher, Li, Li, and
  Fei-Fei]{imagenetcvpr09}
J.~Deng, W.~Dong, R.~Socher, L.-J. Li, K.~Li, and L.~Fei-Fei.
\newblock {ImageNet: A Large-Scale Hierarchical Image Database}.
\newblock In \emph{IEEE Conference on Computer Vision and Pattern Recognition
  (CVPR)}, 2009.

\bibitem[Levine et~al.(2019)Levine, Singla, and Feizi]{levine2019certifiably}
Alexander Levine, Sahil Singla, and Soheil Feizi.
\newblock Certifiably robust interpretation in deep learning.
\newblock \emph{arXiv preprint arXiv:1905.12105}, 2019.

\bibitem[Cao and Gong(2017)]{cao2017mitigating}
Xiaoyu Cao and Neil~Zhenqiang Gong.
\newblock Mitigating evasion attacks to deep neural networks via region-based
  classification.
\newblock \emph{33rd Annual Computer Security Applications Conference}, 2017.

\bibitem[Liu et~al.(2018)Liu, Cheng, Zhang, and Hsieh]{liu2018towards}
Xuanqing Liu, Minhao Cheng, Huan Zhang, and Cho-Jui Hsieh.
\newblock Towards robust neural networks via random self-ensemble.
\newblock In \emph{The European Conference on Computer Vision (ECCV)},
  September 2018.

\bibitem[Zhang and Liang(2019)]{zhang2019discretization}
Y.~Zhang and P.~Liang.
\newblock Defending against whitebox adversarial attacks via randomized
  discretization.
\newblock In \emph{Artificial Intelligence and Statistics (AISTATS)}, 2019.

\bibitem[Pinot et~al.(2019)Pinot, Meunier, Araujo, Kashima, Yger, Gouy-Pailler,
  and Atif]{pinot2019theoretical}
Rafael Pinot, Laurent Meunier, Alexandre Araujo, Hisashi Kashima, Florian Yger,
  Cédric Gouy-Pailler, and Jamal Atif.
\newblock Theoretical evidence for adversarial robustness through
  randomization.
\newblock In \emph{Advances in Neural Information Processing Systems
  (NeurIPS)}, 2019.

\bibitem[Lee et~al.(2019)Lee, Yuan, Chang, and Jaakkola]{lee2019stratified}
Guang-He Lee, Yang Yuan, Shiyu Chang, and Tommi~S. Jaakkola.
\newblock A stratified approach to robustness for randomly smoothed
  classifiers.
\newblock In \emph{Advances in Neural Information Processing Systems
  (NeurIPS)}, 2019.

\end{thebibliography}
